\documentclass[11pt]{article}

\usepackage{microtype}
\usepackage{graphicx}
\usepackage{subcaption}
\usepackage{booktabs} 

\usepackage[margin=1in]{geometry}
\usepackage[authoryear,round,semicolon]{natbib}

\usepackage{algorithm}
\usepackage{algorithmic}

\usepackage{amsmath}
\usepackage{amssymb}
\usepackage{mathtools}
\usepackage{amsthm}
\usepackage[most]{tcolorbox}

\theoremstyle{plain}
\newtheorem{theorem}{Theorem}[section]
\newtheorem{proposition}[theorem]{Proposition}
\newtheorem{lemma}[theorem]{Lemma}
\newtheorem{corollary}[theorem]{Corollary}
\newtheorem{definition}[theorem]{Definition}
\theoremstyle{definition}
\newtheorem{example}[theorem]{Example}

\theoremstyle{remark}
\newtheorem{remark}[theorem]{Remark}

\usepackage[textsize=tiny]{todonotes}

\definecolor{darkblue}{rgb}{0.0,0.0,0.65}
\definecolor{darkred}{rgb}{0.68,0.05,0.0}
\definecolor{darkgreen}{rgb}{0.0,0.29,0.29}
\definecolor{darkpurple}{rgb}{0.47,0.09,0.29}

\usepackage[backref=page]{hyperref}
\hypersetup{
  colorlinks=true,
  citecolor=darkblue,
  linkcolor=darkred,
  filecolor=darkblue,
  urlcolor=darkblue
}

\usepackage[capitalize,noabbrev]{cleveref}

\title{Adaptive Symmetry Discovery for Dynamical System Identification}

\author{
Behrooz Tahmasebi\footnote{Harvard John A. Paulson School of Engineering and Applied Sciences, Harvard University, Cambridge, MA 02138, USA. Emails: \texttt{\{behrooz\_tahmasebi,mweber\}@seas.harvard.edu}}
\and
Melanie Weber\footnotemark[1]
}

\date{}

\begin{document}

\maketitle

\begin{abstract}
Dynamical systems model trajectory data generated by fixed underlying dynamics, with applications ranging from biology to physics.  Especially in scientific settings, dynamical systems are not generic but often exhibit symmetries imposed by physical laws, formalized through equivariance with respect to group actions. The identification problem concerns recovering the parameters of a system from observed trajectories.  In this work, we study \emph{adaptive symmetry discovery} for dynamical system identification and address how a system can be identified from a single trajectory when it is equivariant with respect to an unknown symmetry group. To this end, we first show that for known symmetries, the system can be identified from a significantly shorter single trajectory than in the generic setting, and we precisely characterize this improvement. 
We then consider the automatic symmetry discovery setting, proposing a method to learn the symmetry group directly from a single trajectory and incorporate it into the identification procedure, achieving the same optimal trajectory length as in the known-symmetry case. Our analysis relies on tools from group representation theory and the expander properties of Cayley graphs, and may be of independent interest for the study of symmetries in dynamical systems.
\end{abstract}

\clearpage

\tableofcontents

\section{Introduction}

The classical scientific approach has long relied on discovering fundamental laws of nature expressed through simple and interpretable equations. The recent abundance of data, together with rapid advances in machine learning and deep learning, has given rise to a complementary scientific paradigm: inferring governing laws and equations directly from data using principled methodologies rather than ad hoc modeling assumptions.

Among many modeling frameworks, dynamical systems constitute one of the most fundamental mathematical formalisms for describing flows, trajectory data, and non-independent evolutionary processes in the sciences, with applications ranging from biological systems to physical flows \citep{brunton2016discovering, yu2024learning, wang2025physics}. The problem of discovering governing equations from data, often formulated through the lens of dynamical systems \citep{brunton2016discovering}, has the potential to uncover previously unknown scientific principles and to significantly accelerate scientific discovery.

To this end, a central task is to learn the parameters of a dynamical system from data (observed trajectories), a problem commonly referred to as \emph{dynamical system identification}. As a well-studied topic in system identification, a wide range of approaches and algorithms have been developed to address this problem across various settings \citep{van2012subspace, isermann2011identification}.

In the context of scientific discovery from data, however, dynamical systems are often not generic. Instead, they are structured and frequently exhibit (potentially unknown) symmetries imposed by physical laws, formalized as equivariance with respect to a (possibly unknown) group action. The field of geometric machine learning studies how and when such symmetries can be exploited to improve learning and generalization \citep{bronstein2021geometric, weber2025geometric}. In this work, we focus on the intersection of geometric machine learning and dynamical system identification.

Motivated by these considerations, we study the problem of \emph{adaptive symmetry discovery} for dynamical system identification. Focusing on single-trajectory data and finite groups, the central question we address is how to identify a system when it is equivariant with respect to an unknown symmetry group, and how to automatically discover and exploit this symmetry to improve sample efficiency, namely by enabling identification from shorter trajectories.

First, we show that when the symmetry group is known, the system can be identified from significantly shorter trajectories than in the generic setting, and we precisely characterize this improvement. Second, we turn to the automatic symmetry discovery setting and propose a method to learn the symmetry group directly from a single trajectory and to incorporate it into the identification procedure. Our approach achieves the same optimal trajectory length as in the known-symmetry case. Consequently, under mild conditions, symmetry discovery incurs negligible overhead, allowing one to fully leverage the benefits of equivariance.

It is worth noting that, while a number of recent studies are closely related to our work (as outlined in the related work section), most existing methods are heuristic or model-specific, and provable quantitative guarantees for symmetry discovery in dynamical systems are lacking. In contrast, our focus is on understanding the theoretical and foundational limits of this problem, which, to the best of our knowledge, have been largely unexplored in the literature.

Finally, we emphasize that the tools used in this paper, drawing from group representation theory and the expansion properties of Cayley graphs for finite groups, introduce techniques that are new to this context and may be of independent interest for the study of symmetries in dynamical systems and beyond.

In short, in this paper we make the following contributions:
\begin{itemize}
    \item We study the problem of automatic symmetry discovery for dynamical system identification, with a focus on adapting to unknown finite group invariances.
    \item We show that symmetries can drastically reduce the trajectory length required to identify a system, and that this benefit can be achieved even when the symmetry is unknown by exploiting a new formulation that enables adaptation to the underlying symmetry.
    \item Our analysis leverages tools from group representation theory and the expander properties of Cayley graphs, which may be of independent interest for the broader study of symmetric and equivariant dynamical systems.
\end{itemize}

\section{Related Work}

Recently, the problem of learning dynamical systems from data has received particular attention, driven by applications across the sciences and physics-guided machine learning \citep{brunton2016discovering, yu2024learning, wang2025physics, sivaranjani2025control}. From a theoretical perspective, characterizing when and how a dynamical system can be identified from data is a classical and well-studied problem in system identification \citep{van2012subspace, isermann2011identification}. More recently, significant effort has focused on learning dynamical systems from noisy observations, a setting that goes beyond, and is substantially more challenging than, the noiseless identification regime. For example, \citet{simchowitz2018learning} study the identification of linear dynamical systems from a single noisy trajectory, with subsequent extensions to nonlinear systems under similar single-trajectory assumptions \citep{foster2020learning, ziemann2022single}.

Beyond this, prior work has also investigated finite-sample identification of linear time-invariant systems \citep{sarkar2019near, sarkar2021finite}, regimes involving multiple trajectories \citep{tu2024learning}, as well as system identification and control over a single trajectory~\citep{fefferman2022optimal,carruth2022controlling,carruth2024almost}. In contrast, the focus of this paper is on the identification of dynamical systems \emph{simultaneously} with symmetry discovery. As a first step toward this goal, we focus on the noiseless single-trajectory setting, which allows us to study the fundamental role of symmetry without additional statistical complications.

Although known symmetries are known to provide both empirical and theoretical benefits in learning and estimation, including improved generalization guarantees \citep{tahmasebi2023exact, tahmasebi2023sample}, in many applications the relevant symmetries exist but are not known \emph{a priori}. In such settings, symmetries must be detected or discovered directly from data, as commonly encountered in the discovery of physical laws governed by differential equations. The research direction of \emph{automatic symmetry discovery} seeks to identify underlying symmetries in a principled manner, moving beyond ad hoc or manually engineered approaches.

A wide range of methods for symmetry discovery have been proposed in the literature. Deep learning approaches for symmetry discovery have been explored in several works \citep{desai2022symmetry, yang2023generative, JMLR:v26:24-2175}, including Lie algebra–based convolutional networks \citep{dehmamy2021automatic}; see also \citep{romero2022learning, ko2024learning, hu2025symmetry}. Methods based on infinitesimal generators for discovering symmetries in nonlinear dynamical data have also been studied \citep{pmlr-v267-hu25o}; see \citep{shaw2025cont} for related approaches.

Beyond these, methods for discovering nonlinear group actions in latent spaces have been proposed \citep{yang2024latent}, including approaches that go beyond affine transformations and address manifold-structured data \citep{shaw2024symmetry, bhat2025atlasd}. Another line of work learns symmetries from layer gradients to relax hard invariance constraints \citep{van2023learning}, while flow matching has recently been used to discover Lie group symmetries \citep{park2025discovering}. For symmetry discovery in differential equations and PDEs, see \citep{kreider2025model, hu2025governing, yang2025discovering, yang2024symmetry}. Other approaches include quadratic-form-based methods \citep{karjol2025learning} and learnable data augmentation strategies \citep{santos2025learning}. For symmetry discovery in finite groups using representation-theoretic tools, see \citep{huh2025a}. Jointly discovering and enforcing symmetries has also been studied \citep{otto2025unified}. Finally, we note that symmetry discovery is fundamentally different from (binary) hypothesis testing for the presence of symmetry in data \citep{soleymani2025robust}.

There has also been recent work on symmetry discovery, specifically in dynamical systems. Data-driven detection of Lie point symmetries for continuous dynamical systems is studied in   \citet{gabel2024data}, while the discovery of finite symmetry groups in dynamical systems is considered in \citet{calvo2025learning, calvo2025machine}. A related approach is proposed in \citet{li2025latent}, where the authors introduce latent mixtures of symmetries to model dynamical systems with multiple symmetric latent components.

Finally, we note that a variety of approaches have been proposed to introduce and exploit symmetries in learning, ranging from canonicalization \citep{kaba2023equivariance, tahmasebi2025regularity, tahmasebigeneralization, shumaylovlie} and frame averaging \citep{punyframe, atzmon2022frame, lin2024equivariance} to data augmentation \citep{tahmasebi2025data}. In contrast, this paper introduces an approach based on generating sets of groups and their expansion properties, with close connections to recent work on approximate symmetry \citep{tahmasebi2025achieving} and learning under invariances \citep{soleymanilearning, soleymani2025from}.

\section{Problem Statement}

We first introduce the basic notation and formalize the class of dynamical systems studied in this paper. A detailed review of the required background is deferred to Appendix~\ref{app_back}.



 
\paragraph{Dynamical systems.}
A (discrete-time) dynamical system is specified by a function
\[
f:\mathbb C^d\to\mathbb C^d,
\]
which governs the state evolution according to
\begin{equation}
    x_{t+1}=f(x_t), \qquad t=0,1,\ldots,
\end{equation}
where \(x_t=(x_t^1,x_t^2,\ldots,x_t^d)^\top\in\mathbb C^d\) denotes the \emph{state} of the system at time \(t\).
We refer to \(f\) as the \emph{dynamics} of the system and assume that \(f\) belongs to a function class \(\mathcal F\).

\paragraph{Trajectories and learning objective.}
Given an initial state \(x_0\in\mathbb C^d\), the dynamics \(f\) generates a \emph{trajectory}
\[
(x_0,x_1,\ldots,x_T)\in\mathbb C^{d\times (T+1)},
\]
where \(T\in\mathbb N\) denotes the number of observed transitions.
In the system identification problem, a learner observes such a trajectory and aims to recover the underlying unknown dynamics \(f\in\mathcal F\).

Clearly, if the function class \(\mathcal F\) is too rich, recovering \(f\) from a finite-length trajectory is impossible. This motivates restricting attention to structured and low-complexity families of dynamics.

\subsection{Feature-lifted linear dynamical systems}

In this work, we study a class of nonlinear dynamics that become linear after a finite-dimensional lifting of the state. Let
\[
\Phi:\mathbb C^d\to \mathbb C^m
\]
be an analytic feature map, and consider dynamics of the form
\begin{equation}
    x_{t+1}=f(x_t)=W\Phi(x_t),
    \qquad
    W\in \mathbb C^{d\times m}.
\end{equation}
We refer to \(\Phi\) as the \emph{feature map} and to \(\mathbb C^m\) as the \emph{feature space}. For convenience, we write
\[
    \phi_t \coloneqq \Phi(x_t)\in \mathbb C^m .
\]
Without loss of generality, we assume that $\Phi$ has linearly independent coordinates as functions of $x \in \mathbb{C}^d$.

\begin{definition}[Feature-lifted linear dynamical systems]
Let \(\Phi:\mathbb C^d\to \mathbb C^m\) be a fixed finite-dimensional feature map. We define
\[
\mathcal F_\Phi
\coloneqq
\left\{
f:\mathbb C^d\to \mathbb C^d
\;\middle|\;
f(x)=W\Phi(x) : \forall ~W\in \mathbb C^{d\times m}
\right\}.
\]
Thus, every system in \(\mathcal F_\Phi\) is nonlinear in the state \(x\) in general, but linear in the lifted features \(\Phi(x)\).
\end{definition}

This modeling assumption encompasses a rich family of nonlinear dynamical systems while retaining a linear parameterization in feature space. The choice of \(\Phi\) determines the expressive power of the class, and later sections characterize how the representation-theoretic structure of the feature space controls the trajectory length needed for identifying the system.

\begin{example}[Polynomial features]
A canonical choice is the polynomial feature map consisting of all monomials of total degree at most \(k \in \mathbb{N}\):
\[
\Phi_{\le k}(x)
=
\Bigl(
(x^{1})^{\alpha_1} (x^{2})^{\alpha_2} \cdots (x^{d})^{\alpha_d}
\Bigr)_{\alpha \in \mathcal{I}_k}
\in \mathbb C^m,
\]
where
$
\mathcal{I}_k
=
\left\{
\alpha \in \mathbb Z_{\ge 0}^d
\;\middle|\;
\sum_{i=1}^d \alpha_i \le k
\right\}.
$
The feature dimension is
$
m=\binom{d+k}{d}.
$
We denote the corresponding class by
$
\mathcal F_{\le k}
\coloneqq
\mathcal F_{\Phi_{\le k}}.
$
This recovers polynomial dynamical systems of total degree at most \(k \in \mathbb{N}\).
\end{example}

\begin{example}[Fourier features]
Another canonical choice is a finite band-limited Fourier feature map. Let
$
\Lambda \subset \mathbb Z^d
$
be a finite set of frequencies, for instance
$
\Lambda=\{w\in\mathbb Z^d:\|w\|_2\le C\}
$ for some $C>0$.
Define
\[
\Phi(x)
=
\left(
e^{2\pi i \langle w,x\rangle}
\right)_{w\in\Lambda}
\in \mathbb C^{|\Lambda|}.
\]
The corresponding class \(\mathcal F = \mathcal F_{\Phi}\) consists of dynamics whose coordinates are trigonometric polynomials supported on \(\Lambda\).
\end{example}

\begin{example}[Linear and affine systems]
Linear and affine dynamical systems are included as special cases of polynomial lifting. Indeed, a linear system
\[
x_{t+1}=W_1x_t,
\qquad W_1\in\mathbb C^{d\times d},
\]
and an affine system
\[
x_{t+1}=W_1x_t+w_0,
\qquad W_1\in\mathbb C^{d\times d},\; w_0\in\mathbb C^d,
\]
belong to \(\mathcal F_{\le 1}\), since \(\Phi_{\le 1}\) contains the constant feature and all coordinate functions.
\end{example}

\subsection{Equivariant dynamics}

We now introduce the notion of symmetry in dynamical systems. Intuitively, a dynamical system is symmetric if transformations applied to its input state induce predictable and consistent transformations of its output state. For example, rotating the input state results in a correspondingly rotated output. Such structure reflects underlying invariances often imposed by physical laws and is pervasive in scientific and engineering applications.

In this work, we formalize symmetry through the notion of \emph{equivariance}. We begin by briefly reviewing the necessary group-theoretic background.

\paragraph{Groups and representations.}
A finite group \(G\) is a finite set equipped with an associative binary operation, an identity element, and inverses for all elements. Common examples include finite rotation groups, reflection groups, and the permutation group on \(n\) elements, consisting of all bijections \(\sigma:[n]\to[n]\) under composition.

A linear representation of a group \(G\) on \(\mathbb C^n\) is a homomorphism \(\rho:G\to \mathrm{GL}_n(\mathbb C)\), assigning to each \(g\in G\) an invertible matrix \(\rho(g)\in\mathbb C^{n\times n}\) such that
\[
\rho(gh)=\rho(g)\rho(h), \qquad \forall g,h\in G.
\]
We refer to \(\rho(g)\) as the action of \(g\) on \(\mathbb C^n\).

\paragraph{Lifted group actions.}
Suppose a finite group \(G\) acts linearly on the state space \(\mathbb C^d\) via a representation \(\rho:G\to\mathrm{GL}_d(\mathbb C)\). We assume that the feature map \(\Phi:\mathbb C^d\to\mathbb C^m\) is compatible with this action, in the sense that there exists a representation \(\rho_\Phi:G\to\mathrm{GL}_m(\mathbb C)\) satisfying
\begin{equation}
    \Phi(\rho(g)x)=\rho_\Phi(g)\Phi(x),
    \qquad \forall g\in G,\; x\in\mathbb C^d.
\end{equation}
Thus, the group action on the state space induces a corresponding linear action on the feature space. This compatibility holds for the canonical feature maps considered in this paper, including polynomial features and finite Fourier features whenever the feature dictionary is closed under the action of \(G\).

We are now ready to define equivariant dynamical systems.

\begin{definition}[Equivariant dynamical systems]
Let \(G\) be a finite group acting on \(\mathbb C^d\) via a representation \(\rho\). A dynamical system \(f:\mathbb C^d\to\mathbb C^d\) is said to be \emph{\(G\)-equivariant} if
\begin{equation}
    f(\rho(g)x)=\rho(g)f(x),
    \qquad \forall g\in G,\; x\in\mathbb C^d.
\end{equation}
For feature-lifted linear dynamical systems \(f(x)=W\Phi(x)\in\mathcal F_\Phi\), \(G\)-equivariance is equivalently enforced by the intertwining condition
\[
    \rho(g)W=W\rho_\Phi(g),
    \qquad \forall g\in G.
\]
We denote by \(\mathcal F_\Phi^G\) the class of \(G\)-equivariant feature-lifted linear dynamical systems with feature map \(\Phi\).
\end{definition}

This condition expresses the commutation of the parameter matrix \(W\) with the group actions on the state and feature spaces, respectively.

\begin{remark}
Throughout the paper, we use the terms \emph{symmetric} and \emph{equivariant} interchangeably when referring to dynamical systems. The term equivariant is used when emphasizing the underlying group action.
\end{remark}

\subsection{Identifiability from a single trajectory}

We now formalize the notion of identifiability.

\begin{definition}[Generic identifiability]
Fix a feature map \(\Phi:\mathbb C^d\to\mathbb C^m\), a finite group \(G\), and consider the class \(\mathcal F_\Phi^G\). The dynamics within this class are said to be \emph{generically identifiable from trajectories of length \(T\in\mathbb N\)} if, for almost all initial states \(x_0\in\mathbb C^d\) and almost all \(G\)-equivariant parameter matrices \(W\in\mathbb C^{d\times m}\), the corresponding trajectory
\[
(x_0,x_1,\ldots,x_T)\in\mathbb C^{d\times (T+1)}
\]
uniquely determines \(W\). The minimal such \(T\) is denoted by \(T_\Phi(G)\).
\end{definition}

Here, ``generic'' is understood in the standard sense: identifiability holds outside a set of measure zero in the space of initial states and \(G\)-equivariant parameters. This notion allows us to exclude pathological configurations while retaining full generality.

When identifiability holds for trajectories of length \(T\), we informally refer to \(T\) as the \emph{sample complexity} of the identification problem, since each transition provides one observation of the dynamics.

\begin{remark}
The notation \(T_\Phi(G)\) emphasizes that the trajectory length depends not only on the symmetry group \(G\), but also on the feature map \(\Phi\). Later, we characterize this dependence through the representation-theoretic decomposition of the feature space and the generic excitation rank induced by \(\Phi\).
\end{remark}

\paragraph{Symmetry discovery.}
Consider an unknown \(G\)-equivariant dynamical system
\(f:\mathbb C^d\to\mathbb C^d\), where the underlying symmetry group
\(G\) is unknown. We assume, however, that \(G\) belongs to a known class of
admissible finite groups \(\mathcal{G}\). While the specific group
\(G \in \mathcal{G}\) is not known a priori, prior knowledge of the class
\(\mathcal{G}\) provides structural information that can be leveraged for
system identification.

We are primarily interested in settings where \(\mathcal{G}\) consists of
relatively large finite groups, as such symmetries can lead to substantial
reductions in the trajectory length required for identifiability. At the
same time, the unknown identity of \(G\) introduces a nontrivial challenge,
as the learner must simultaneously identify the dynamics and discover the
underlying symmetry.

The goal of \emph{adaptive symmetry discovery} for dynamical system
identification is to recover the dynamics \(f\) from short trajectories,
using only the assumption that \(f\in\mathcal F_\Phi\) is equivariant with respect to some
unknown group \(G \in \mathcal{G}\). In the ideal case, the best sample
complexity one could hope to achieve is
\[
    T_\Phi(\mathcal{G}) \coloneqq \max_{G \in \mathcal{G}} T_\Phi(G),
\]
corresponding to the worst-case identifiability threshold over the class
\(\mathcal{G}\).

A further challenge arises from computational considerations. Finite
groups of interest are often prohibitively large. For example, the permutation group on \(n\) elements has cardinality
\(n! \approx \exp(n \log n)\), as do many of its subgroups.
Consequently, algorithms whose runtime scales linearly with the group size
are infeasible, and one must instead aim for procedures with runtime at most 
polylogarithmic in the group size.
Accordingly, our aim is to address the following question:
\begin{tcolorbox}
\emph{Given a class of groups $\mathcal{G}$ and an unknown dynamical system
$f:\mathbb{C}^d \to \mathbb{C}^d$ with $f \in \mathcal{F}_\Phi$ that is
$G$-equivariant for an unknown $G \in \mathcal{G}$, can one identify the
dynamics from trajectories of length $T_\Phi(\mathcal{G})$ with efficient
computational runtime?}
\end{tcolorbox}

We answer this question affirmatively in the next section.

\section{Main Results}
\label{sec:main_results}

We first characterize the trajectory length required for equivariant system
identification when the symmetry group \(G\) is known, and then turn to
adaptive symmetry discovery.

\subsection{Sample complexity of equivariant identification}
\label{sec:equivariant_identification}

We study \(T_\Phi(G)\), the minimal trajectory length required to generically
identify a \(G\)-equivariant system in \(\mathcal F_\Phi^G\), when \(G\) and
\(\Phi\) are known. We briefly recall the required notation; see Appendix~\ref{app_back}.

\paragraph{Irreducible representations.}
Let \(\widehat G\) denote the set of equivalence classes of irreducible
representations of \(G\). After a change of basis, the state and feature spaces
decompose as
\[
\mathbb C^d
\cong
\bigoplus_{\pi\in\widehat G}\mathbb C^{n_\pi}\otimes V_\pi,
\qquad
\mathbb C^m
\cong
\bigoplus_{\pi\in\widehat G}\mathbb C^{m_\pi}\otimes V_\pi,
\]
where \(d_\pi=\dim V_\pi\), and \(n_\pi,m_\pi \in \mathbb{N} \) denote the multiplicities of
\(\pi\) in the state and feature representations. Thus
\(d=\sum_\pi n_\pi d_\pi\) and \(m=\sum_\pi m_\pi d_\pi\).

Under this decomposition, every \(G\)-equivariant matrix
\(W:\mathbb C^m\to\mathbb C^d\) decomposes as
\[
W=
\bigoplus_{\pi\in\widehat G} C_\pi\otimes I_{V_\pi},
\qquad
C_\pi\in\mathbb C^{n_\pi\times m_\pi}.
\]
Hence identifying \(W\) reduces to identifying \(C_\pi\) for all
\(\pi\) with \(n_\pi>0\).

\paragraph{Generic excitation rank.}
The required trajectory length also depends on how \(\Phi\) excites the
different isotypic components. The \(\pi\)-isotypic component of \(\Phi(x)\)
belongs to \(\mathbb C^{m_\pi}\otimes V_\pi\). Fixing bases
\(\{v_i\}_{i=1}^{m_\pi}\) and \(\{u_j\}_{j=1}^{d_\pi}\), we write it as
\[
\sum_{i=1}^{m_\pi}\sum_{j=1}^{d_\pi}
[\Phi_\pi(x)]_{ij}\,v_i\otimes u_j,
\]
thereby identifying it with
\(\Phi_\pi(x)\in\mathbb C^{m_\pi\times d_\pi}\). Thus, the rows index the
\(m_\pi\) copies of \(\pi\), while the columns correspond to coordinates in
\(V_\pi\).

For a trajectory \(x_0,\ldots,x_T\), define
\[
\mathbf{\Phi}_{\pi,T}
\coloneqq
[\Phi_\pi(x_0),\ldots,\Phi_\pi(x_{T-1})]
\in\mathbb C^{m_\pi\times Td_\pi},
\]
and let
\(h_{\pi,\Phi}(T)\coloneqq
\operatorname{rank}_{\mathrm{gen}}(\mathbf{\Phi}_{\pi,T})\),
where the generic rank is taken over the initial state and the equivariant
parameter matrix \(W\).

\begin{theorem}[Equivariant system identification]
\label{thm:equivariant_identification}
Fix a finite group \(G\) and an analytic feature map
\(\Phi:\mathbb C^d\to\mathbb C^m\) with linearly independent coordinates.
Then \(W\in\mathcal F_\Phi^G\) is generically identifiable from
\(x_0,\ldots,x_T\) if and only if \(h_{\pi,\Phi}(T)=m_\pi\) for every
\(\pi\) with \(n_\pi>0\). Consequently,
\[
T_\Phi(G)=
\max_{\substack{\pi:n_\pi>0\\ m_\pi>0}}
\inf\{T\in\mathbb N:h_{\pi,\Phi}(T)=m_\pi\},
\]
with the infimum interpreted as \(+\infty\) if full rank is never attained.
Irreducible components with \(m_\pi=0\) contain no unknown parameters and
therefore impose no identification requirement.
\end{theorem}

Indeed, in each \(\pi\)-block the observed transitions have the form
\(C_\pi\Phi_\pi(x_t)\). Stacking the observations therefore gives a linear
system with design matrix \(\mathbf{\Phi}_{\pi,T}\), so \(C_\pi\) is uniquely
determined exactly when \(\mathbf{\Phi}_{\pi,T}\) has full row rank \(m_\pi\).

Since
\(\operatorname{rank}(\mathbf{\Phi}_{\pi,T})
\le\min\{m_\pi,Td_\pi\}\), we obtain the universal lower bound
\[
T_\Phi(G)\ge
T_{\mathrm{rep}}(G)
\coloneqq
\max_{\pi:n_\pi>0}
\left\lceil\frac{m_\pi}{d_\pi}\right\rceil.
\]
This is the smallest trajectory length allowed by dimension counting; whether
it is attained depends on the excitation profile \(h_{\pi,\Phi}(T)\).

For the trivial group, there is a single one-dimensional irrep with
multiplicity \(m\). Under our assumptions on \(\Phi\),
\(h_{\pi,\Phi}(T)=\min\{m,T\}\) generically, and hence we have
\(T_\Phi(\{e\})=m\).

\paragraph{Linear and affine systems.}
For linear systems, \(\Phi(x)=x\), so \(m_\pi=n_\pi\). For every finite group,
\[
h_{\pi,\Phi}(T)=\min\{n_\pi,Td_\pi\},
\quad
T_{\mathrm{lin}}(G)=
\max_{\pi:n_\pi>0}
\left\lceil\frac{n_\pi}{d_\pi}\right\rceil.
\]
Thus the representation-theoretic lower bound is always tight. The affine case
follows similarly after adjoining the constant feature, which adds one copy of
the trivial representation to the feature space.

\paragraph{Finite Abelian groups.}
Suppose \(G\) is finite Abelian and \(\Phi\) is analytic, has linearly
independent coordinates, and contains the state coordinates, i.e., there is
\(L\in\mathbb C^{d\times m}\) such that \(L\Phi(x)=x\). Since every
irreducible complex representation of \(G\) is one-dimensional,
\(d_\pi=1\) for all \(\pi\in\widehat G\). In this case, $h_{\pi,\Phi}(T)=\min\{m_\pi,T\}$ and
\[
T_\Phi(G)=\max_{\pi:n_\pi>0}m_\pi=T_{\mathrm{rep}}(G).
\]
In particular, this result holds for the full polynomial feature maps
\(\Phi_{\le k}\).

\paragraph{Permutation-equivariant polynomial systems.}
Let \(G\) be a subgroup of the symmetric group \(S_d\), acting by coordinate
permutations, and let \(\Phi=\Phi_{\le k}\) contain all monomials of total
degree at most \(k\), with \(k\ge2\). Then, we show
\[
T_{\Phi_{\le k}}(G)
\le
\max_{\pi:n_\pi>0}m_\pi.
\]
Thus the trajectory length is controlled by the largest active representation
multiplicity rather than by the ambient feature dimension
\(m=\sum_\pi m_\pi d_\pi\).

\begin{example}[Quadratic systems with permutation symmetry]
\label{ex:quadratic_permutation}
Consider \(\Phi=\Phi_{\le2}\) under the standard \(S_d\)-action, where
\(d\ge4\). Let \(V_i\) denote the space of homogeneous polynomials of degree
\(i\). Then \(\mathcal P_{\le2}=V_0\oplus V_1\oplus V_2\), with
\(V_0=\pi_0\), \(V_1=\pi_0\oplus\pi_{\mathrm{std}}\), and
\(V_2=2\pi_0\oplus2\pi_{\mathrm{std}}\oplus\pi_{(d-2,2)}\).
Here \(\pi_0\) is the trivial representation,
\(d_{\pi_{\mathrm{std}}}=d-1\), and
\(d_{\pi_{(d-2,2)}}=d(d-3)/2\). Hence
\(m_{\pi_0}=4\), \(m_{\pi_{\mathrm{std}}}=3\), and
\(m_{\pi_{(d-2,2)}}=1\).

The state representation is
\(\mathbb C^d\cong\pi_0\oplus\pi_{\mathrm{std}}\), so only the trivial and
standard blocks are relevant. Their excitation ranks satisfy
\(h_{\pi_0,\Phi}(T)=\min\{4,T\}\) and
\(h_{\pi_{\mathrm{std}},\Phi}(T)=\min\{3,2T\}\), yielding
\[
T_{\Phi_{\le2}}(S_d)=4.
\]
By contrast, \(m=\binom{d+2}{2}=\Theta(d^2)\). Thus permutation equivariance
reduces the generic trajectory length from \(\Theta(d^2)\) to a constant
independent of \(d\).
\end{example}

\paragraph{Representation stability.}
The preceding phenomenon extends to any fixed polynomial degree. Under the
standard permutation action of \(S_d\), the state representation decomposes as
\[
\mathbb C^d\cong\pi_0\oplus\pi_{\mathrm{std}},
\]
so only the trivial and standard irreducible representations are relevant for
identification, even though other irreducible representations may appear in the
feature space. For fixed \(k\), the multiplicities of these two representations
in \(\Phi_{\le k}\) stabilize  stabilize once \(d\ge k+1\).   In particular, for \(d\)
sufficiently large,
\[
m_{\pi_0}=\sum_{r=0}^{k}p(r),\qquad
m_{\pi_{\mathrm{std}}}
=\sum_{j=0}^{k-1}(k-j)p(j),
\]
where \(p(r)\) denotes the integer partition function.
Defining
\[
M_k\coloneqq
\max\{m_{\pi_0},m_{\pi_{\mathrm{std}}}\},
\]
the preceding polynomial identification bound gives
\(T_{\Phi_{\le k}}(S_d)\le M_k\). Hence, for every fixed \(k \ge 2\),
\[
T_{\Phi_{\le k}}(S_d)=O_k(1),
\quad
\dim(\Phi_{\le k})=\binom{d+k}{k}=\Theta_k(d^k).
\]
Thus, fixed-degree permutation-equivariant polynomial systems are generically
identifiable from trajectories whose length remains bounded independently of
the state dimension \(d\).

\subsection{Adaptive symmetry discovery}
\label{sec:adaptive}

We now consider system identification when the underlying symmetry is unknown.
Let \(\Gamma\) be a finite ambient group of transformations, and let
\(\mathcal G\) be a finite family of candidate subgroups of \(\Gamma\). We
assume that \(\Gamma\) acts on the state and feature spaces through compatible
representations \(\rho\) and \(\rho_\Phi\), with the actions of each candidate
group obtained by restriction. We further assume that the full symmetry group
of the unknown dynamics within \(\Gamma\) belongs to \(\mathcal G\).

We first recall the notion of a generating set, which provides a compact
description of a potentially large finite group.

\begin{definition}[Generating set]
A subset \(S\subseteq G\) is called a \emph{generating set} of a finite group
\(G\) if every element of \(G\) can be obtained through finitely many
compositions of elements of \(S\) and their inverses. Equivalently, every
\(g\in G\) can be written as
\(g=s_1^{\epsilon_1}\cdots s_r^{\epsilon_r}\), where
\(s_i\in S\) and \(\epsilon_i\in\{-1,1\}\). In this case, we write
\(G=\langle S\rangle\).
\end{definition}

Generating sets are closely connected to Cayley graphs. Given \(S\subseteq G\),
the Cayley graph \(\operatorname{Cay}(G,S)\) has vertex set \(G\), with edges
corresponding to composition by elements of \(S\) and their inverses. This
graph is connected if and only if \(S\) generates \(G\). The classical
Alon--Roichman theorem~\citep{alon1994random} establishes the substantially
stronger fact that \(O(\log |G|)\) uniformly sampled group elements produce an
expanding random Cayley graph with high probability. In particular, they form
a generating set with high probability.

For our purposes, expansion itself is not required, and a simpler argument
already yields the logarithmic generation bound. As long as the sampled
elements generate a proper subgroup, a new uniform sample lies outside this
subgroup with probability at least \(1/2\). Whenever this occurs, the size of
the generated subgroup at least doubles. Let
\[
|G|_{\max}\coloneqq\max_{H\in\mathcal G}|H|
\]
and define
\[
N_{\mathcal G,\delta}
\coloneqq
\left\lceil
8\left(
\left\lceil\log_2|G|_{\max}\right\rceil
+\log\frac{|\mathcal G|}{\delta}
\right)
\right\rceil.
\]
Then \(N_{\mathcal G,\delta}\) i.i.d.\ uniform samples from each candidate
group generate all groups in \(\mathcal G\) simultaneously with probability
at least \(1-\delta\). We provide the elementary argument in the appendix.

\paragraph{Separating candidate symmetries.}
Known-symmetry identifiability alone does not necessarily imply that different
candidate symmetry classes can be distinguished from the same short
trajectory. We therefore isolate the condition required for adaptive
discovery.

\begin{definition}[Generic candidate separation]
\label{def:candidate_separation}
Let
\[
T_\Phi(\mathcal G)
\coloneqq
\max_{G\in\mathcal G}T_\Phi(G).
\]
We call \(\mathcal G\) \emph{generically separating} if, for every
\(G\in\mathcal G\), a generic \(G\)-equivariant system whose full symmetry
group within \(\Gamma\) is \(G\), together with a generic initial state, has
the following property: for every \(T\ge T_\Phi(\mathcal G)\), if a candidate
\(H\in\mathcal G\) admits an \(H\)-equivariant system consistent with the
observed trajectory, then \(H\) is a subgroup of \(G\).
\end{definition}

Thus, generic separation rules out spurious candidate groups that are
incomparable with, or strictly larger than, the true group. It does not rule
out subgroups of the true group, since a \(G\)-equivariant system is
automatically equivariant with respect to every subgroup of \(G\).

The condition is automatic when the candidate groups are totally ordered by
inclusion. Indeed, every candidate subgroup of the true group is feasible.
Conversely, if a candidate group strictly contains the true group, then every
system equivariant to that candidate is also equivariant to the true group.
Once \(T\ge T_\Phi(G)\), uniqueness within the \(G\)-equivariant class forces
such a system to coincide with the true dynamics, contradicting the assumption
that \(G\) is its full symmetry group.

We are now ready to state the adaptive identification procedure.

\begin{algorithm}[t]
\caption{Adaptive symmetry discovery}
\label{alg:adaptive}
\begin{algorithmic}[1]
\STATE \textbf{Input:} trajectory \((x_0,\ldots,x_T)\), feature map
\(\Phi\), candidate family \(\mathcal G\), failure probability \(\delta\)
\STATE \textbf{Output:} parameter matrix \(W\) and generators \(S\) of the
discovered symmetry group
\STATE Form
\(X\gets[\Phi(x_0),\ldots,\Phi(x_{T-1})]\) and
\(Y\gets[x_1,\ldots,x_T]\)
\STATE Set \(N\gets N_{\mathcal G,\delta}\) and
\(\mathcal C\gets\varnothing\)
\FOR{each \(G\in\mathcal G\)}
    \STATE Sample \(S_G=\{g_1,\ldots,g_N\}\), where
    \(g_i\overset{\mathrm{i.i.d.}}{\sim}\operatorname{Unif}(G)\)
    \STATE Test feasibility of
    \[
    WX=Y,\qquad
    \rho(g)W=W\rho_\Phi(g)
    \quad\text{for all }g\in S_G
    \]
    \IF{\textbf{feasible}}
        \STATE Store \((G,S_G,W_G)\) in \(\mathcal C\), where \(W_G\)
        is any feasible solution
    \ENDIF
\ENDFOR
\STATE Choose a candidate \(G\) of maximum cardinality in \(\mathcal C\)
\STATE \textbf{return} \(W_G,S_G\)
\end{algorithmic}
\end{algorithm}

\begin{theorem}[Adaptive symmetry discovery]
\label{thm:adaptive}
Let \(\mathcal G\) be a generically separating family of finite candidate
groups, and suppose that the full symmetry group \(G\) of an unknown system
\(f(x)=W\Phi(x)\) belongs to \(\mathcal G\). Suppose the trajectory is
generated from a generic initial state by a generic \(G\)-equivariant system.
Then, provided that
\[
T\ge T_\Phi(\mathcal G)
=
\max_{H\in\mathcal G}T_\Phi(H),
\]
Algorithm~\ref{alg:adaptive} recovers the unknown dynamics and a generating
set of its full symmetry group with probability at least \(1-\delta\).
Assuming uniform sampling and representation evaluation can be performed
efficiently for each candidate group, the runtime is polynomial in
\(d,m,T,|\mathcal G|,\log|G|_{\max}\), and
\(\log(1/\delta)\).
\end{theorem}

The theorem shows that, under generic candidate separation, discovering the
symmetry requires no additional trajectory observations beyond those needed
when the group is known. The additional cost is computational: the algorithm
solves linearly constrained feasibility problems using only logarithmically
many sampled elements per candidate group. From the Cayley-graph viewpoint,
these samples define a sparse random Cayley graph of each candidate. The
Alon--Roichman theorem guarantees the stronger property of expansion using
\(O(\log|G|)\) samples, while our identification procedure only requires
connectivity, or equivalently generation.

\begin{remark}[Why the largest feasible group is selected]
Suppose \(G\) is the full symmetry group of the unknown dynamics. With
probability at least \(1-\delta\), the sampled set \(S_G\) generates \(G\), so
imposing equivariance with respect to \(S_G\) is equivalent to imposing
equivariance with respect to all of \(G\). Since \(T\ge T_\Phi(G)\), the
corresponding feasible system is uniquely the true one.

Every candidate subgroup of \(G\) is also feasible. Generic separation rules
out all other candidates. Hence \(G\) is the unique feasible candidate of
maximum cardinality.
\end{remark}

\begin{remark}[Randomization and random Cayley graphs]
The randomness in Algorithm~\ref{alg:adaptive} is used only to obtain compact
generating sets for the candidate groups. The sampled elements may
equivalently be viewed as defining random Cayley graphs. By the
Alon--Roichman theorem, \(O(\log|G|)\) random elements suffice with high
probability to obtain an expanding Cayley graph, and hence a generating set.
Our analysis requires only the weaker connectivity property and therefore uses
the elementary subgroup-growth argument above. If a generating set is supplied
for every candidate group in advance, this randomization is unnecessary.
\end{remark}

\begin{remark}[Nested candidate families]
If the candidate groups are totally ordered by inclusion, generic candidate
separation is automatic. Thus adaptive discovery over a nested family achieves
the same trajectory length \(T_\Phi(\mathcal G)\) as identification with the
group known in advance.
\end{remark}

\begin{remark}[Optimal trajectory length]
Under generic candidate separation, Algorithm~\ref{alg:adaptive} achieves the
known-symmetry trajectory length. In particular, no additional trajectory
observations are required for discovering the unknown group. The additional
randomness and computation are used only to determine which symmetry
constraints are compatible with the observed dynamics.
\end{remark}

\begin{remark}[General candidate families]
For arbitrary incomparable candidate groups,
\(T_\Phi(\mathcal G)\) need not be sufficient for symmetry discovery, even
though it suffices once the correct group is known. In this case, one may
instead consider the smallest trajectory horizon at which generic candidate
separation holds. The same algorithm and argument then apply at this larger
horizon.
\end{remark}

\paragraph{Discovery among bounded-index subgroups.}
We next consider a complementary setting in which the possible symmetries are
large subgroups of a single known ambient group. Let \(\Gamma\) be a finite
group and define
\[
\mathcal H_B
:=
\{H\le \Gamma:[\Gamma:H]\le B\}.
\]
We assume that the full symmetry group \(H\) of the unknown dynamics belongs
to \(\mathcal H_B\), but we do not require an enumeration of
\(\mathcal H_B\).

The key observation is that a uniformly sampled element of \(\Gamma\) belongs
to \(H\) with probability
\[
\mathbb P_{g\sim\operatorname{Unif}(\Gamma)}(g\in H)
=
\frac{|H|}{|\Gamma|}
=
\frac{1}{[\Gamma:H]}
\ge \frac1B.
\]
Hence, if membership in the unknown symmetry group can be determined from the
observed trajectory, sampling from \(\Gamma\) provides uniform samples from
\(H\) without explicitly searching over the candidate subgroups.

To formalize this, we assume \emph{generic elementwise separation}: for a
generic system with full symmetry group \(H\), a generic initial state, and
the trajectory length under consideration, a sampled \(g\in\Gamma\) satisfies
\(g\in H\) if and only if there exists \(\widetilde W\) such that
\[
\widetilde W X=Y,
\qquad
\rho(g)\widetilde W
=
\widetilde W\rho_\Phi(g).
\]
The reverse implication is automatic for \(g\in H\), since the true parameter
matrix \(W\) satisfies both constraints.

Therefore, we may sample \(g\sim\operatorname{Unif}(\Gamma)\), retain it
whenever the above feasibility test succeeds, and repeat. Conditioned on
acceptance, the retained elements are i.i.d.\ uniform samples from \(H\).
Since the acceptance probability is at least \(1/B\), obtaining \(N\)
accepted elements requires at most \(BN\) ambient samples in expectation.

This again admits a Cayley-graph interpretation. The accepted elements define
a random Cayley graph of the unknown group \(H\). Thus
\(O(\log|H|)\) accepted elements suffice to generate \(H\) with high
probability, while the Alon--Roichman theorem gives the stronger expansion
guarantee.

Define
\[
N_{\Gamma,\delta}
:=
\left\lceil
8\left(
\left\lceil\log_2|\Gamma|\right\rceil
+
\log\frac1\delta
\right)
\right\rceil.
\]

\begin{algorithm}[t]
\caption{Discovery of a bounded-index symmetry group}
\label{alg:bounded_index}
\begin{algorithmic}[1]
\STATE \textbf{Input:} trajectory \((x_0,\ldots,x_T)\), feature map
\(\Phi\), ambient group \(\Gamma\), index bound \(B\), failure probability
\(\delta\)
\STATE \textbf{Output:} parameter matrix \(W\) and generators \(S\) of the
discovered subgroup
\STATE Form
\(X\gets[\Phi(x_0),\ldots,\Phi(x_{T-1})]\) and
\(Y\gets[x_1,\ldots,x_T]\)
\STATE Set \(N\gets N_{\Gamma,\delta}\) and \(S\gets\varnothing\)
\WHILE{\(|S|<N\)}
    \STATE Sample \(g\sim\operatorname{Unif}(\Gamma)\)
    \STATE Test feasibility of
    \[
    \widetilde W X=Y,
    \qquad
    \rho(g)\widetilde W=\widetilde W\rho_\Phi(g)
    \]
    \IF{\textbf{feasible}}
        \STATE Add \(g\) to \(S\)
    \ENDIF
\ENDWHILE
\STATE Solve
\[
WX=Y,
\qquad
\rho(g)W=W\rho_\Phi(g)
\quad\text{for all }g\in S
\]
\STATE \textbf{return} \(W,S\)
\end{algorithmic}
\end{algorithm}

\begin{theorem}[Bounded-index subgroup discovery]
\label{cor:bounded_index}
Suppose the full symmetry group \(H\) of the unknown dynamics satisfies
\([\Gamma:H]\le B\), and suppose generic elementwise separation holds for
\[
T\ge
T_\Phi(\mathcal H_B)
:=
\max_{\substack{H\le\Gamma\\
[\Gamma:H]\le B}}
T_\Phi(H).
\]
Then Algorithm~\ref{alg:bounded_index} recovers the unknown dynamics and a
generating set of \(H\) with probability at least \(1-\delta\).

The algorithm does not enumerate the subgroups in \(\mathcal H_B\). The
expected number of ambient samples and feasibility tests is
\[
O\left(
B\left(
\log|\Gamma|+\log\frac1\delta
\right)
\right),
\]
and is therefore independent of \(|\mathcal H_B|\). Assuming efficient
sampling from \(\Gamma\), representation evaluation, and linear feasibility
testing, the expected runtime is polynomial in
\(d,m,T,B,\log|\Gamma|\), and \(\log(1/\delta)\). In particular, when
\(B=O(1)\), the sampling overhead is logarithmic in \(|\Gamma|\).
\end{theorem}

\begin{remark}[Role of the index bound]
The index bound controls the rejection-sampling overhead. Since
\([\Gamma:H]\le B\), each uniform ambient sample belongs to the unknown group
with probability at least \(1/B\). Thus bounded index allows the group to be
discovered directly from ambient samples without constructing a separate
sampler or enumerating candidate subgroups.
\end{remark}

\begin{remark}[Computational efficiency]
Each sampled element requires only a linear feasibility test in the entries of
\(\widetilde W\), consisting of the trajectory constraint
\(\widetilde W X=Y\) and one intertwining constraint
\(\rho(g)\widetilde W=\widetilde W\rho_\Phi(g)\). Hence the number of such
tests depends on \(B\) and \(\log|\Gamma|\), but not on the potentially very
large number of bounded-index subgroups.
\end{remark}

Therefore, the main message of this section can be summarized as follows:
\begin{tcolorbox}
When the candidate symmetries are generically distinguishable, equivariant
dynamics can be identified without knowing the symmetry in advance and with
the same trajectory length as in the known-symmetry setting. Moreover, when
the unknown symmetry is a bounded-index subgroup of a known ambient group,
there is no need to enumerate the candidate subgroups: one can instead sample
directly from the ambient group, test individual elements for symmetry, and
recover the unknown subgroup from the accepted samples. The resulting
sampling overhead is at most a factor \(B\), and is constant when
\(B=O(1)\).
\end{tcolorbox}

\section{Conclusion and Future Work}

In this paper, we study how to integrate adaptive symmetry discovery with the problem of dynamical system identification. Focusing on single-trajectory data, we address the question of how to identify a system that is symmetric with respect to an unknown finite group. Our main contribution is a method based on the theory of Cayley graph expanders and generating sets of finite groups, which enables adaptation to unknown symmetries with near-zero overhead. Moreover, the proposed approach achieves the same trajectory length (i.e., sample complexity) as in the setting where the symmetries are known, thereby yielding optimal adaptation.

An important direction for future work is to extend our results to infinite (Lie) groups, which would likely require theoretical tools beyond the expander-based framework developed for finite groups. Another promising direction is to study noisy dynamical systems and to identify system parameters while simultaneously adapting to symmetries in such settings. Establishing provable sample complexity guarantees in the presence of noise remains open, to the best of our knowledge. More broadly, it would be interesting to investigate whether similar adaptive symmetry techniques can be developed for learning linear time-invariant systems and for control systems with hidden states and noisy observations in modern settings \citep{hazan2025research}. We leave these directions for future work.

\section*{Acknowledgements}

BT and MW were partially supported by NSF Award CBET-2112085 and DMS-2406905. MW acknowledges partial funding from an Alfred P. Sloan Fellowship in Mathematics and the AI2050 program at Schmidt Sciences (Grant G-25-69786).
This material is based on research sponsored by the Air Force Office of Scientific Research under agreement number FA9550261B044. The U.S. Government is authorized to reproduce and distribute reprints for Governmental purposes notwithstanding any copyright notation thereon. The opinions, findings, views, conclusions or recommendations contained herein are those of the authors and should not be interpreted as necessarily representing the official policies or endorsements, either expressed or implied, of the DAF, AFRL or the U.S. Government.
 

\clearpage
\bibliography{ref}
\bibliographystyle{plainnat}

\newpage
\appendix
\onecolumn














\section{Preliminaries}\label{app_back}

In this section, we summarize the basic definitions and background material required to understand the results of the paper. Standard references for finite group theory and representation theory include \citep{serre1977linear,fulton2013representation}.

\subsection{Groups, representations, and equivariance}

\paragraph{Groups.}
A (finite) group $G$ is a finite set equipped with a binary operation $\cdot : G \times G \to G$ satisfying the following properties:
\begin{itemize}
    \item \textbf{Associativity:} $(g \cdot h) \cdot s = g \cdot (h \cdot s)$ for all $g,h,s \in G$.
    \item \textbf{Identity:} There exists an element $e \in G$ such that $e \cdot g = g \cdot e = g$ for all $g \in G$.
    \item \textbf{Inverses:} For every $g \in G$, there exists $g^{-1} \in G$ such that $g \cdot g^{-1} = g^{-1} \cdot g = e$.
\end{itemize}
The cardinality of a finite group $G$ is denoted by $|G|$. For convenience, we drop the dot notation and write $gh$ instead of $g \cdot h$ for all $g,h \in G$.

\paragraph{Examples of finite groups.} Here is a few examples of finite groups:
\begin{itemize}
    \item The cyclic group of integers modulo $n$ under addition, denoted by $\mathbb{Z}/n\mathbb{Z}$.
    \item The permutation group (also known as symmetric group) of group of all permutations of $n$ elements.
    \item Direct products of commutative groups, such as group of sign inversions $\{ \pm 1\}^d$
    \item Dihedral group, consisting of symmetries of a regular polygon with $n$ sides in dimension two (reflections, rotations).
\end{itemize}

\paragraph{Group actions and linear representations.}
Let $V$ be a vector space over $\mathbb{C}$. A (left) \emph{action} of a finite group $G$ on $V$ is a map
\[
\theta : G \times V \to V
\]
satisfying
\[
\theta(e,v) = v, \qquad 
\theta(gh, v) = \theta(g, \theta(h, v))
\quad \text{for all } g,h \in G, \ v \in V.
\]
When $\theta(g,\cdot)$ is linear and invertible for every $g \in G$, the action is equivalently described by a group homomorphism
\[
\rho : G \to \mathrm{GL}(V),
\]
where $\rho(g)v := \theta(g,v)$. In this case, $(V,\rho)$ is called a \emph{linear representation} of $G$. Note that here each $\rho(g) \in \mathbb{C}^{n \times n}$ is an invertible matrix, with $n = \dim(V)$.

\emph{Note:} While group actions can be defined on general sets or manifolds, throughout this paper we only consider linear actions on finite-dimensional complex vector spaces.

\paragraph{Irreducible representations.}
A representation $(V,\rho)$ of $G$ is said to be \emph{irreducible} if it has no nontrivial $G$-invariant subspaces, i.e., the only subspaces $W \subseteq V$ satisfying $\rho(g)W \subseteq W$ for all $g \in G$ are $\{0\}$ and $V$ itself.

A fundamental result in finite group representation theory is \emph{Maschke’s theorem}, which states that every finite-dimensional representation of a finite group over $\mathbb{C}$ is completely reducible. That is, any representation $V$ admits a decomposition of the form
\[
V \;\cong\; \bigoplus_{\pi \in \widehat{G}} \mathbb{C}^{n_\pi} \otimes V_\pi,
\]
where:
\begin{itemize}
    \item $\widehat{G}$ denotes the set of inequivalent irreducible representations of $G$,
    \item $V_\pi$ is a representative irreducible representation of dimension $d_\pi \in \mathbb{N}$,
    \item $n_\pi \in \mathbb{Z}_{\ge 0}$ is the multiplicity of $\pi$ in $V$.
\end{itemize}
The dimensions satisfy
\[
\sum_{\pi \in \widehat{G}} n_\pi d_\pi = n.
\]

Moreover, for finite groups, the number of irreducible representations is finite, thus $|\widehat{G}|<\infty$, and their dimensions obey the following identity
\[
\sum_{\pi \in \widehat{G}} d_\pi^2 = |G|,
\]
with each $d_\pi$ dividing $|G|$.

\paragraph{Unitary representations and change of basis.}
For finite groups, every finite-dimensional representation over $\mathbb{C}$ is equivalent to a unitary representation. In particular, after an appropriate change of basis, we may assume without loss of generality that $\rho(g)$ and $\rho_\Phi(g)$ are unitary for all $g \in G$. This change of basis does not affect equivariance properties or the structure of the state/feature space.

\paragraph{Isotypic decomposition and block structure.}
The decomposition above can be made explicit at the level of matrices. In particular, there exists a change of basis of $V$ under which the representation $\rho$ takes a block-diagonal form
\[
\rho(g) \;=\; \bigoplus_{\pi \in \widehat{G}} \left( I_{n_\pi} \otimes \pi(g) \right),
\qquad \forall g \in G,
\]
where each $\pi(g) \in \mathbb{C}^{d_\pi \times d_\pi}$ is an irreducible representation and $I_{n_\pi}$ denotes the identity matrix of size $n_\pi$.
We use the notation
\[
\rho \;\cong\; \bigoplus_{\pi \in \widehat{G}} n_\pi \, \pi
\]
to denote the decomposition of $\rho$ into irreducible representations, where $n_\pi$ denotes the multiplicity of $\pi$.

\paragraph{Equivariant linear maps.}
Let $(V,\rho_V)$ and $(U,\rho_U)$ be two representations of $G$. A linear map $\psi : V \to U$ is called \emph{$G$-equivariant} if
\[
\psi \circ \rho_V(g) = \rho_U(g) \circ \psi
\quad \text{for all } g \in G.
\]
The space of equivariant linear maps between $V$ and $U$ is denoted by $\mathrm{Hom}_G(V,U)$.

Using the decomposition into irreducible representations, the structure of $\mathrm{Hom}_G(V,U)$ can be characterized explicitly in terms of multiplicities. In particular, if
\[
V \cong \bigoplus_{\pi \in \widehat{G}} \mathbb{C}^{n_\pi} \otimes V_\pi,
\qquad
U \cong \bigoplus_{\pi \in \widehat{G}} \mathbb{C}^{m_\pi} \otimes V_\pi,
\]
then
\[
\mathrm{Hom}_G(V,U) \cong \bigoplus_{\pi \in \widehat{G}} \mathbb{C}^{m_\pi \times n_\pi},
\]
and in particular,
\[
\dim (\mathrm{Hom}_G(V,U)) = \sum_{\pi \in \widehat{G}} m_\pi n_\pi.
\]
This characterization will play a central role in our analysis of equivariant linear dynamical systems. In particular, it allows us to count the dimension of the matrices $W \in \mathbb{R}^{d \times m}$ that satisfy the equivariance condition in our proofs.

\paragraph{Polynomial feature lifting.}
Fix a degree parameter $k \in \mathbb{N}$. Let
\[
\Phi : \mathbb{C}^d \to \mathbb{C}^m
\]
denote the polynomial feature map consisting of all monomials in $d$ variables of total degree at most $k$. Explicitly, for $x = (x^1,\dots,x^d)^\top \in \mathbb{C}^d$,
\[
\Phi(x)
\;=\;
\Bigl(
(x^{1})^{\alpha_1} (x^{2})^{\alpha_2} \cdots (x^{d})^{\alpha_d}
\Bigr)_{\alpha \in \mathcal{I}_k},
\qquad
\mathcal{I}_k
=
\Bigl\{
\alpha \in \mathbb{Z}_{\ge 0}^d
\;\big|\;
\sum_{i=1}^d \alpha_i \le k
\Bigr\}.
\]
The resulting feature dimension is
\[
m = \binom{d+k}{d}.
\]
We refer to $\Phi(x)$ as the \emph{feature representation} of the state $x$.

\paragraph{Polynomially lifted linear dynamical systems.}
A \emph{polynomially lifted linear dynamical system} of degree at most $k$ is a function
\[
f : \mathbb{C}^d \to \mathbb{C}^d
\]
of the form
\[
f(x) = W \Phi(x),
\qquad W \in \mathbb{C}^{d \times m}.
\]
We denote by $\mathcal{F}_{\le k}$ the class of all such systems. Although $f$ is generally nonlinear in the state $x$, it is linear in the lifted feature space. Thus, each system in $\mathcal{F}_{\le k}$ is fully parameterized by the matrix $W$.

For convenience, along a trajectory $\{x_t\}_{t \ge 0}$, we write $\phi_t := \Phi(x_t) \in \mathbb{C}^m$.

\paragraph{Lifted group actions on polynomial features.}
Suppose a finite group $G$ acts linearly on the state space $\mathbb{C}^d$ via a representation
\[
\rho : G \to \mathrm{GL}_d(\mathbb{C}).
\]
This action induces a natural linear action on the polynomial feature space associated with $\Phi$. In particular, there exists a unique representation
\[
\rho_\Phi : G \to \mathrm{GL}_m(\mathbb{C})
\]
such that
\begin{equation}
\Phi(\rho(g)x) = \rho_\Phi(g)\,\Phi(x),
\qquad \forall g \in G,\; x \in \mathbb{C}^d.
\label{eq:lifted_action}
\end{equation}
The representation $\rho_\Phi$ corresponds to the action of $G$ on multivariate polynomials of degree at most $k$ induced by the change of variables $x \mapsto \rho(g)x$.

\paragraph{Equivariant polynomially lifted systems.}
Let $G$ act on $\mathbb{C}^d$ via $\rho$. A dynamical system
$f : \mathbb{C}^d \to \mathbb{C}^d$
is said to be \emph{$G$-equivariant} if
\[
f(\rho(g)x) = \rho(g) f(x),
\qquad \forall g \in G,\; x \in \mathbb{C}^d.
\]
For polynomially lifted linear systems $f(x) = W \Phi(x) \in \mathcal{F}_{\le k}$, this condition is equivalent to the matrix constraint
\begin{equation}
\rho(g) W = W \rho_\Phi(g),
\qquad \forall g \in G,
\label{eq:equivariance_constraint}
\end{equation}
which states that $W$ is an intertwining operator between the representations $\rho_\Phi$ and $\rho$.

We denote by $\mathcal{F}_{\le k}^G \subseteq \mathcal{F}_{\le k}$ the class of all $G$-equivariant polynomially lifted linear dynamical systems of degree at most $k$.

\paragraph{Equivariance as linear constraints.}
For fixed representations $\rho$ and $\rho_\Phi$ of a finite group $G$, the equivariance condition
\[
\rho(g) W = W \rho_\Phi(g), \qquad \forall g \in G,
\]
is a system of linear constraints in the entries of $W$. Consequently, the space of $G$-equivariant matrices forms a linear subspace of $\mathbb{C}^{d \times m}$ whose dimension can be characterized using representation-theoretic multiplicities given above. This observation is essential for the study of the time complexity of optimization problems considered later in the paper.

\paragraph{Genericity and analytic maps.}
Throughout the paper, a property is said to hold \emph{generically} if it
holds outside a set of Lebesgue measure zero in the natural finite-dimensional
parameter space. In the polynomial setting, the exceptional set can typically
be taken to be a proper algebraic variety; for analytic models, it is contained
in the zero set of a nonzero analytic function.

Importantly, genericity is always understood with respect to the underlying
parameters, such as the initial state \(x_0\) and the parameter matrix \(W\),
rather than with respect to the ambient feature space. In particular, when
\(m>d\), the feature vector \(\Phi(x)\) need not be generic as an arbitrary
element of \(\mathbb C^m\), since the image of \(\Phi\) may itself lie in a
lower-dimensional subset.

The main genericity principle used throughout our proofs concerns ranks of
analytic matrix-valued functions. Let \(M(\theta)\) be a matrix whose entries
depend analytically on a finite-dimensional parameter \(\theta\), and let
\[
r=\max_\theta \operatorname{rank} M(\theta).
\]
Then \(\operatorname{rank}M(\theta)=r\) generically. Indeed, some
\(r\times r\) minor is not identically zero, and its zero set has Lebesgue
measure zero. We refer to \(r\) as the \emph{generic rank} of \(M\).

For the dynamical systems considered in this paper, each trajectory point
\(x_t\) is an analytic function of the initial state and the parameter matrix:
starting from \(x_{t+1}=W\Phi(x_t)\), this follows recursively from the
analyticity of \(\Phi\). Consequently, all feature design matrices constructed
from a finite trajectory have entries that depend analytically on
\((x_0,W)\). Generic rank statements can therefore be proved by exhibiting
a single choice of \((x_0,W)\) for which an appropriate minor is nonzero.
This is the genericity argument used repeatedly in the proofs below.

\subsection{Cayley graphs, expanders, and random generators}
\label{app:cayley}

Let \(G\) be a finite group and let
\(S=\{s_1,\ldots,s_N\}\) be a multiset of elements of \(G\). We write
\(S^{\pm}\) for the symmetric multiset obtained by adjoining the inverses of
the elements of \(S\). The Cayley graph
\(\operatorname{Cay}(G,S^{\pm})\) has vertex set \(G\), with edges
corresponding to composition by elements of \(S\) and their inverses. Its
connected components are precisely the cosets of the subgroup
\(\langle S\rangle\). In particular,
\[
\operatorname{Cay}(G,S^{\pm})\text{ is connected}
\quad\Longleftrightarrow\quad
\langle S\rangle=G.
\]

The normalized adjacency operator of this graph is
\[
\mathcal A_S f(g)
=
\frac{1}{2N}\sum_{i=1}^N
\bigl(f(s_i g)+f(s_i^{-1}g)\bigr),
\qquad
f:G\to\mathbb C.
\]
Since the generating multiset is symmetric, \(\mathcal A_S\) is self-adjoint
on \(\ell_2(G)\).

\paragraph{Representation-theoretic decomposition.}
The regular representation decomposes as
\[
\ell_2(G)
\cong
\bigoplus_{\pi\in\widehat G}
\mathbb C^{d_\pi}\otimes V_\pi,
\]
where \(d_\pi=\dim V_\pi\). Choosing the irreducible representations to be
unitary, the operator \(\mathcal A_S\) decomposes, up to multiplicity, into
the matrices
\[
A_\pi(S)
\coloneqq
\frac{1}{2N}\sum_{i=1}^N
\bigl(\pi(s_i)+\pi(s_i)^\ast\bigr).
\]
Thus the nontrivial spectrum of the Cayley graph is controlled by the matrices
\(A_\pi(S)\) for \(\pi\neq\mathbf 1\).

A spectral gap is a stronger property than what is needed in our algorithms.
Indeed, the following immediate observation connects expansion to generation.

\begin{lemma}[Spectral certificate for generation]
\label{lem:spectral_generation}
If
\[
\max_{\pi\neq\mathbf 1}
\|A_\pi(S)\|_{\mathrm{op}}<1,
\]
then \(\langle S\rangle=G\).
\end{lemma}

\begin{proof}
If \(\langle S\rangle\neq G\), the Cayley graph has more than one connected
component. Hence the eigenvalue \(1\) of its normalized adjacency operator has
multiplicity greater than one. One copy corresponds to the constant functions,
while another lies in the orthogonal complement of the constants. Under the
irreducible decomposition of the regular representation, this yields a
nontrivial \(\pi\) for which \(A_\pi(S)\) has eigenvalue \(1\), contradicting
the assumed operator-norm bound.
\end{proof}

The classical Alon--Roichman theorem~\citep{alon1994random} gives a much
stronger probabilistic statement: for any fixed desired spectral gap,
\(O(\log |G|)\) independent uniform samples from \(G\) suffice to produce an
expanding random Cayley graph, with constants depending only on the desired
gap. In particular, logarithmically many random elements generate \(G\) with
high probability.

For our purposes, however, expansion is not needed. The following elementary
subgroup-growth argument gives directly the generation guarantee used in
\cref{sec:adaptive}.

\begin{proposition}[Random generation]
\label{prop:random_generators}
Let \(g_1,\ldots,g_N\overset{\mathrm{i.i.d.}}{\sim}\operatorname{Unif}(G)\)
and \(S=\{g_1,\ldots,g_N\}\). Let
\(r=\lceil\log_2|G|\rceil\). If
\[
N\ge
8\left(
r+\log\frac{1}{\delta}
\right),
\]
then
\[
\mathbb P\bigl(\langle S\rangle\neq G\bigr)\le\delta.
\]
\end{proposition}

\begin{proof}
Let \(H_i=\langle g_1,\ldots,g_i\rangle\), with \(H_0=\{e\}\). Whenever
\(H_{i-1}\) is a proper subgroup of \(G\), Lagrange's theorem gives
\(|H_{i-1}|\le |G|/2\), and therefore
\[
\mathbb P
\bigl(
g_i\notin H_{i-1}
\,\big|\,
g_1,\ldots,g_{i-1}
\bigr)
\ge\frac12.
\]
Whenever \(g_i\notin H_{i-1}\), the subgroup strictly grows, and hence
\(|H_i|\ge2|H_{i-1}|\). Consequently, after \(r\) such successful
enlargements, the generated subgroup must equal \(G\).

The number of successful enlargements therefore stochastically dominates a
binomial random variable \(X\sim\operatorname{Bin}(N,1/2)\) until generation
is complete. Hence
\[
\mathbb P(\langle S\rangle\neq G)
\le
\mathbb P(X<r).
\]
Writing \(\mu=N/2\), the assumption on \(N\) implies \(r\le\mu/4\). A
multiplicative Chernoff bound therefore gives
\[
\mathbb P(X<r)
\le
\mathbb P\left(X\le\frac{\mu}{4}\right)
\le
\exp\left(-\frac{9\mu}{32}\right)
=
\exp\left(-\frac{9N}{64}\right)
\le\delta.
\]
\end{proof}

The preceding proposition immediately yields the simultaneous guarantee used
by Algorithm~\ref{alg:adaptive}.

\begin{corollary}[Simultaneous random generation]
\label{cor:simultaneous_generators}
Let \(\mathcal G\) be a finite family of finite groups and define
\[
|G|_{\max}
\coloneqq
\max_{G\in\mathcal G}|G|.
\]
If, independently for every \(G\in\mathcal G\), we draw
\(N_{\mathcal G,\delta}\) i.i.d.\ uniform samples, where
\[
N_{\mathcal G,\delta}
=
\left\lceil
8\left(
\left\lceil\log_2|G|_{\max}\right\rceil
+
\log\frac{|\mathcal G|}{\delta}
\right)
\right\rceil,
\]
then all sampled sets generate their corresponding groups simultaneously with
probability at least \(1-\delta\).
\end{corollary}

\begin{proof}
Apply Proposition~\ref{prop:random_generators} to each candidate group with failure
probability \(\delta/|\mathcal G|\), and take a union bound.
\end{proof}

\begin{remark}[Relation to Alon--Roichman]
The proof of Proposition~\ref{prop:random_generators} uses only connectivity of the
random Cayley graph. The Alon--Roichman theorem gives the considerably
stronger conclusion that a logarithmic-size random generating set typically
produces a Cayley graph with a nontrivial spectral gap. Thus, the random
generators used in our adaptive procedure can be viewed as a weaker
consequence of the random-Cayley-graph expansion phenomenon.
\end{remark}

\paragraph{Rejection sampling for bounded-index subgroups.}
We finally record the elementary sampling fact used in the bounded-index
variant of our algorithm.

\begin{lemma}[Sampling a bounded-index subgroup]
\label{lem:rejection_sampling}
Let \(H\) be a subgroup of a finite group \(\Gamma\) satisfying
\([\Gamma:H]\le B\). Suppose that we can sample uniformly from \(\Gamma\)
and test membership in \(H\). If
\(g\sim\operatorname{Unif}(\Gamma)\) is accepted whenever \(g\in H\), then
the accepted sample is exactly uniform on \(H\). Moreover, the acceptance
probability satisfies
\[
\mathbb P(g\in H)
=
\frac{|H|}{|\Gamma|}
=
\frac{1}{[\Gamma:H]}
\ge\frac{1}{B}.
\]
Consequently, obtaining \(N\) i.i.d.\ uniform samples from \(H\) requires at
most \(BN\) samples from \(\Gamma\) in expectation.
\end{lemma}

\begin{proof}
For every \(h\in H\),
\[
\mathbb P(g=h\mid g\in H)
=
\frac{1/|\Gamma|}{|H|/|\Gamma|}
=
\frac1{|H|},
\]
so the conditional distribution is uniform on \(H\). Independent repetitions
of the rejection procedure therefore produce i.i.d.\ uniform samples. The
acceptance probability is \(1/[\Gamma:H]\), so the expected number of ambient
samples per accepted element is \([\Gamma:H]\le B\).
\end{proof}

\begin{remark}[High-probability sampling complexity]
The expected bound in Lemma~\ref{lem:rejection_sampling} can also be upgraded to
a high-probability bound using a standard Chernoff argument. In particular,
\(O(B(N+\log(1/\delta)))\) ambient samples suffice to obtain \(N\) accepted
samples with probability at least \(1-\delta\). Thus, when \(B=O(1)\), both
the expected and high-probability sampling overheads are constant up to the
logarithmic failure-probability term.
\end{remark}

\section{Proofs of Main Results}
We provide the proof of the main theorems in the paper here in this section.

\subsection{Proof of \cref{thm:equivariant_identification}}
\label{app:proof_identification}

We first prove the general rank characterization and then establish the
special cases used in the main text.

\begin{proof}[Proof of \cref{thm:equivariant_identification}]
Recall the isotypic decompositions
\[
\mathbb C^d\cong
\bigoplus_{\pi\in\widehat G}\mathbb C^{n_\pi}\otimes V_\pi,
\qquad
\mathbb C^m\cong
\bigoplus_{\pi\in\widehat G}\mathbb C^{m_\pi}\otimes V_\pi .
\]
By Schur's lemma, every \(G\)-equivariant matrix
\(W:\mathbb C^m\to\mathbb C^d\) has the form
\[
W=\bigoplus_{\pi\in\widehat G} C_\pi\otimes I_{V_\pi},
\qquad
C_\pi\in\mathbb C^{n_\pi\times m_\pi}.
\]

For notational convenience, reshape the \(\pi\)-isotypic component of a state
\(x\) as a matrix \(X_\pi(x)\in\mathbb C^{n_\pi\times d_\pi}\). With the
analogous reshaping
\(\Phi_\pi(x)\in\mathbb C^{m_\pi\times d_\pi}\) of the feature vector, the
transition \(x_{t+1}=W\Phi(x_t)\) gives, for every \(\pi\),
\[
X_\pi(x_{t+1})=C_\pi\Phi_\pi(x_t).
\]
Stacking \(T\) transitions yields
\[
[X_\pi(x_1),\ldots,X_\pi(x_T)]
=
C_\pi\mathbf{\Phi}_{\pi,T},
\]
where
\(\mathbf{\Phi}_{\pi,T}
=[\Phi_\pi(x_0),\ldots,\Phi_\pi(x_{T-1})]\).

Thus, for a fixed observed trajectory, \(C_\pi\) is uniquely determined if
and only if \(\mathbf{\Phi}_{\pi,T}\) has full row rank \(m_\pi\). Indeed, if
it has full row rank, it admits a right inverse, which uniquely determines
\(C_\pi\). Conversely, if its rank is smaller than \(m_\pi\), there exists a
nonzero matrix \(D_\pi\in\mathbb C^{n_\pi\times m_\pi}\) such that
\(D_\pi\mathbf{\Phi}_{\pi,T}=0\). Replacing \(C_\pi\) by
\(C_\pi+D_\pi\) then gives a distinct equivariant parameter producing exactly
the same observed transitions.

It remains only to pass from a fixed trajectory to the generic statement. The
entries of \(\mathbf{\Phi}_{\pi,T}\) are analytic functions of \(x_0\) and
the equivariant parameters \(C_\pi\). Hence its maximal attainable rank is
achieved outside the common zero set of its maximal nonzero minors, a
measure-zero set. Therefore this maximal rank is precisely
\(h_{\pi,\Phi}(T)\). Since the blocks are independent, \(W\) is generically
identifiable exactly when
\[
h_{\pi,\Phi}(T)=m_\pi
\qquad
\text{for every }\pi\text{ with }n_\pi>0.
\]
Taking the smallest such \(T\) for every active block and then the largest
over the active irreducible representations gives
\[
T_\Phi(G)=
\max_{\pi:n_\pi>0}
\inf\{T\in\mathbb N:h_{\pi,\Phi}(T)=m_\pi\}.
\]
Finally,
\(\operatorname{rank}(\mathbf{\Phi}_{\pi,T})
\le \min\{m_\pi,Td_\pi\}\), and therefore
\[
T_\Phi(G)\ge
\max_{\pi:n_\pi>0}
\left\lceil\frac{m_\pi}{d_\pi}\right\rceil .
\]
\end{proof}

\begin{remark}[Identification may fail for every finite \(T\)]
The possibility \(T_\Phi(G)=+\infty\) in
\cref{thm:equivariant_identification} can genuinely occur without further
assumptions on \(\Phi\). For example, let \(G=\mathbb Z/2\mathbb Z\) act on
\((u,v)\in\mathbb C^2\) by \((u,v)\mapsto(u,-v)\), and consider
\[
\Phi(u,v)=(1,v,uv,u^2v).
\]
The coordinates of \(\Phi\) are analytic and linearly independent. Every
equivariant system in this class has the form
\[
u_{t+1}=a,\qquad
v_{t+1}=v_t(b_0+b_1u_t+b_2u_t^2).
\]
After the first transition, \(u_t=a\), and hence all subsequent feature
vectors in the nontrivial isotypic component are proportional to
\((1,a,a^2)\). Together with the initial feature vector, their rank is at
most two, whereas this component has multiplicity three. Thus its parameter
block can never be identified from a single trajectory.
\end{remark}

\paragraph{The trivial group.}
We next justify the non-equivariant baseline \(T_\Phi(\{e\})=m\). Write
\(\Phi=(\phi_1,\ldots,\phi_m)\). Since the coordinate functions are linearly
independent,
\[
\operatorname{span}\{\Phi(x):x\in\mathbb C^d\}=\mathbb C^m;
\]
otherwise a nonzero linear functional annihilating this span would give a
nontrivial linear relation among the \(\phi_i\). We may therefore choose
\(z_0,\ldots,z_{m-1}\) such that
\(\Phi(z_0),\ldots,\Phi(z_{m-1})\) are linearly independent.

Choose any \(z_m\in\mathbb C^d\). Since these \(m\) feature vectors form a
basis of \(\mathbb C^m\), there exists \(W\) satisfying
\(W\Phi(z_t)=z_{t+1}\) for \(0\le t<m\). Hence
\(z_0,\ldots,z_m\) is a valid trajectory whose feature design matrix has rank
\(m\). The corresponding determinant is therefore a nonzero analytic
function of \((W,x_0)\), so it is nonzero generically. Since a trajectory of
length \(T<m\) provides only \(T\) feature vectors, we conclude that
\[
T_\Phi(\{e\})=m.
\]

\paragraph{Linear systems.}
Suppose now that \(\Phi(x)=x\). Then \(m_\pi=n_\pi\), and in the
\(\pi\)-isotypic component the dynamics become
\[
X_{\pi,t+1}=C_\pi X_{\pi,t},
\qquad
X_{\pi,t}\in\mathbb C^{n_\pi\times d_\pi}.
\]
Consequently,
\[
\mathbf{\Phi}_{\pi,T}
=
[X_{\pi,0},C_\pi X_{\pi,0},\ldots,C_\pi^{T-1}X_{\pi,0}],
\]
which is a block Krylov matrix.

We claim that generically
\[
\operatorname{rank}(\mathbf{\Phi}_{\pi,T})
=
\min\{n_\pi,Td_\pi\}.
\]
It suffices to exhibit one choice attaining this rank. Set
\(n=n_\pi\) and \(r=d_\pi\), and choose
\(C_\pi=\operatorname{diag}(\lambda_1,\ldots,\lambda_n)\), with distinct
nonzero \(\lambda_i\). Partition
\(\min\{n,Tr\}\) rows into at most \(r\) groups, each of size at most \(T\),
and for every row in the \(j\)-th group choose the corresponding row of
\(X_{\pi,0}\) to be the \(j\)-th standard basis vector of \(\mathbb C^r\).
After permuting rows and columns, the resulting full-rank minor consists of
Vandermonde blocks
\[
\begin{bmatrix}
1 & \lambda_i & \cdots & \lambda_i^{T-1}
\end{bmatrix}.
\]
Each block has maximal rank because the corresponding \(\lambda_i\)'s are
distinct. Hence the total rank is \(\min\{n,Tr\}\). Since the relevant minors
are polynomial in \(C_\pi\) and \(X_{\pi,0}\), the same rank holds
generically. Therefore
\[
h_{\pi,\Phi}(T)=\min\{n_\pi,Td_\pi\},
\]
and
\[
T_{\mathrm{lin}}(G)
=
\max_{\pi:n_\pi>0}
\left\lceil\frac{n_\pi}{d_\pi}\right\rceil .
\]

\paragraph{Affine systems.}
The same argument extends to affine dynamics after adjoining the constant
feature. For every nontrivial irrep, the dynamics and the proof above are
unchanged. In the trivial component, write
\(z_{t+1}=Az_t+b\), where \(z_t\in\mathbb C^{n_0}\). The associated feature
vector is \((1,z_t)\). Choosing \(A\) diagonal with distinct eigenvalues
different from \(1\), and writing \(z_\star=(I-A)^{-1}b\), gives
\(z_t=z_\star+A^t(z_0-z_\star)\). After an invertible row operation, the
feature design matrix has rows
\[
1,\qquad
y_i,\lambda_i y_i,\ldots,\lambda_i^{T-1}y_i,
\]
with \(y=z_0-z_\star\). This is again a Vandermonde system, now with nodes
\(1,\lambda_1,\ldots,\lambda_{n_0}\). Thus the same
representation-theoretic formula holds using the multiplicities of the affine
feature representation; the constant feature simply adds one copy of the
trivial representation.

\paragraph{Finite Abelian groups.}
We first record a simple analytic lemma.

\begin{lemma}[Analytic Vandermonde lemma]
\label{lem:analytic_vandermonde}
Let \(f_1,\ldots,f_r:\mathbb C^d\to\mathbb C\) be linearly independent
analytic functions. Then there exist \(x\in\mathbb C^d\) and a diagonal matrix
\(A=\operatorname{diag}(\lambda_1,\ldots,\lambda_d)\) such that
\[
\det\bigl[f_i(A^t x)\bigr]_
{i=1,\ldots,r;\,t=0,\ldots,r-1}\ne0.
\]
\end{lemma}

\begin{proof}
Expand the \(f_i\)'s into their Taylor series at the origin. By invertible row
operations, we may choose a basis of their span with distinct initial
monomials \(x^{\alpha_1},\ldots,x^{\alpha_r}\). Choose a positive weight
vector \(w\) for which these initial monomials remain distinct, and write
\(e_i=\langle w,\alpha_i\rangle\). Let
\(x(s)=(s^{w_1},\ldots,s^{w_d})\). Then, for suitable nonzero coefficients
\(c_i\),
\[
f_i(A^t x(s))
=
c_i(\lambda^{\alpha_i})^t s^{e_i}
+\text{higher-order terms in }s.
\]
Choose the diagonal entries of \(A\) generically so that the numbers
\(\lambda^{\alpha_1},\ldots,\lambda^{\alpha_r}\) are distinct. The lowest
order term of the determinant is then
\[
\left(\prod_{i=1}^r c_i\right)
s^{\sum_i e_i}
\prod_{1\le i<j\le r}
(\lambda^{\alpha_j}-\lambda^{\alpha_i}),
\]
which is nonzero. Hence the determinant is not identically zero, and is
nonzero for some sufficiently small nonzero \(s\).
\end{proof}

We now prove the Abelian claim from the main text. Since \(G\) is finite
Abelian, there is a basis of \(\mathbb C^d\) in which every \(\rho(g)\) is
diagonal. Hence every diagonal matrix \(A\) commutes with the group action.
By assumption, the state coordinates are contained in the span of the feature
coordinates, so there is \(L\) satisfying \(L\Phi(x)=x\). Consequently, every
equivariant linear map \(x\mapsto Ax\) of the above diagonal form belongs to
\(\mathcal F_\Phi^G\).

Fix an active irrep \(\pi\). Since all irreducible complex representations of
a finite Abelian group are one-dimensional, \(d_\pi=1\), and
\(\Phi_\pi(x)\) consists of \(m_\pi\) scalar analytic functions. These
functions are linearly independent because the coordinates of \(\Phi\) are
linearly independent. Applying Lemma~\ref{lem:analytic_vandermonde} with
\(r=m_\pi\) gives a valid equivariant linear trajectory for which
\(\mathbf{\Phi}_{\pi,m_\pi}\) has full rank. Hence this rank is attained
generically. Moreover, the first \(T\) columns of a nonsingular
\(m_\pi\times m_\pi\) design matrix are independent, giving
\[
h_{\pi,\Phi}(T)=\min\{m_\pi,T\}.
\]
Therefore
\[
T_\Phi(G)
=
\max_{\pi:n_\pi>0}m_\pi
=
\max_{\pi:n_\pi>0}\frac{m_\pi}{d_\pi}
=
T_{\mathrm{rep}}(G).
\]

\paragraph{Permutation-equivariant polynomial systems.}
We next prove the upper bound for full polynomial features. The key ingredient
is the following elementary polynomial-orbit lemma.

\begin{lemma}[Polynomial orbit lemma]
\label{lem:polynomial_orbit}
Let \(g_1,\ldots,g_r\in\mathbb C[x_1,\ldots,x_d]\) be linearly independent
polynomials and let \(q\ge2\). Then there exists \(x\in\mathbb C^d\) such
that
\[
\det
\bigl[g_i(x^{\odot q^t})\bigr]_
{i=1,\ldots,r;\,t=0,\ldots,r-1}
\ne0,
\]
where the power is applied coordinatewise.
\end{lemma}

\begin{proof}
Because the \(g_i\)'s contain only finitely many monomials, choose a positive
integer weight vector \(w\) assigning distinct weights to all monomials
appearing in them. After invertible row operations, we may assume that their
lowest-weight monomials are distinct, say
\(c_i x^{\alpha_i}\), with weights
\(e_1<\cdots<e_r\).

Set \(x(s)=(s^{w_1},\ldots,s^{w_d})\). Then
\[
g_i(x(s)^{\odot q^t})
=
c_i s^{e_iq^t}
+\text{higher-order terms}.
\]
In the determinant expansion, the lowest power of \(s\) is obtained by
pairing the ordered numbers \(e_i\) with the oppositely ordered numbers
\(q^t\). By the strict rearrangement inequality, this minimizing permutation
is unique. Its coefficient is \(\pm\prod_i c_i\ne0\). Hence the determinant
is a nonzero polynomial in \(s\), and is nonzero for some \(s\).
\end{proof}

Now let \(G\) be any subgroup of \(S_d\), acting by coordinate permutations,
and let \(\Phi=\Phi_{\le k}\) with \(k\ge2\). Fix an active irrep \(\pi\), and
write the \(\pi\)-isotypic feature component in terms of its \(m_\pi\)
multiplicity copies as
\[
F_1(x),\ldots,F_{m_\pi}(x)\in V_\pi.
\]
These are linearly independent polynomial equivariant maps.

Choose any nonzero linear functional \(\ell\in V_\pi^*\), and define
\(g_i=\ell\circ F_i\). The polynomials \(g_1,\ldots,g_{m_\pi}\) are linearly
independent. Indeed, if
\(\sum_i a_i g_i=0\), then the equivariant map
\(F=\sum_i a_iF_i\) has image contained in the proper subspace
\(\ker\ell\). If \(F\ne0\), the linear span of its image is a nonzero
\(G\)-invariant subspace of the irreducible space \(V_\pi\), and hence must
equal \(V_\pi\), a contradiction. Thus \(F=0\), and linear independence of
the \(F_i\)'s gives \(a_i=0\).

Consider now the coordinatewise squaring dynamics
\[
x_{t+1}=x_t^{\odot2}.
\]
This map is equivariant under every coordinate-permutation action and belongs
to \(\mathcal F_{\Phi_{\le k}}\) whenever \(k\ge2\). By
Lemma~\ref{lem:polynomial_orbit}, there is \(x_0\) for which
\[
[g_i(x_0),g_i(x_1),\ldots,g_i(x_{m_\pi-1})]_{i=1}^{m_\pi}
\]
has rank \(m_\pi\). This matrix is obtained from
\(\mathbf{\Phi}_{\pi,m_\pi}\) by applying the same functional \(\ell\) to
the \(V_\pi\)-coordinate at every time step. Hence
\(\mathbf{\Phi}_{\pi,m_\pi}\) itself has full row rank. The corresponding
minor is therefore not identically zero in the system parameters and initial
state, so full rank holds generically. We conclude that
\[
T_{\Phi_{\le k}}(G)
\le
\max_{\pi:n_\pi>0}m_\pi.
\]

\paragraph{Quadratic systems with permutation symmetry.}
We finally verify the claims in
Example~\ref{ex:quadratic_permutation}. Let \(V_r\) denote the homogeneous
degree-\(r\) polynomial space. Clearly \(V_0=\pi_0\), while \(V_1\) is the
permutation representation, so
\(V_1=\pi_0\oplus\pi_{\mathrm{std}}\).

For \(V_2\), the diagonal monomials
\(\{x_i^2\}_{i=1}^d\) form another copy of the permutation representation,
while the off-diagonal monomials
\(\{x_ix_j:i<j\}\) form the permutation representation on two-element
subsets. For \(d\ge4\), the latter decomposes as
\(\pi_0\oplus\pi_{\mathrm{std}}\oplus\pi_{(d-2,2)}\). Hence
\[
V_2
=
2\pi_0\oplus2\pi_{\mathrm{std}}
\oplus\pi_{(d-2,2)},
\]
giving \(m_{\pi_0}=4\), \(m_{\pi_{\mathrm{std}}}=3\), and
\(m_{\pi_{(d-2,2)}}=1\).

Only \(\pi_0\) and \(\pi_{\mathrm{std}}\) appear in the state representation.
For the trivial block, \(d_{\pi_0}=1\), so
\(h_{\pi_0,\Phi}(T)\le\min\{4,T\}\). The polynomial-orbit argument above
provides a trajectory attaining rank four at \(T=4\); its first \(T\)
columns are then independent for every \(T\le4\). Thus
\[
h_{\pi_0,\Phi}(T)=\min\{4,T\}.
\]

For the standard block, let
\(P=I-\frac1d\mathbf 1\mathbf 1^\top\) denote projection onto the standard
representation. A convenient basis for its three multiplicity copies is
\[
F_1(x)=Px,\qquad
F_2(x)=P(x^{\odot2}),\qquad
F_3(x)=(\mathbf 1^\top x)Px.
\]
For every fixed \(x\), \(F_3(x)\) is proportional to \(F_1(x)\), so the
one-step rank is at most two. Since \(F_1(x)\) and \(F_2(x)\) are generically
independent, \(h_{\pi_{\mathrm{std}},\Phi}(1)=2\).

It remains to show that two transitions generically give rank three. Use the
squaring dynamics and take
\[
x_0=(1,2,3,0,\ldots,0).
\]
Represent \(V_{\mathrm{std}}\) using coordinate differences relative to the
last coordinate. Selecting the first two such coordinates at time \(0\) and
the first at time \(1\) gives the minor
\[
\begin{pmatrix}
1&2&1\\
1&4&1\\
6&12&14
\end{pmatrix},
\]
whose determinant equals \(16\). Thus the rank is three at \(T=2\), and hence
generically. Therefore
\[
h_{\pi_{\mathrm{std}},\Phi}(T)=\min\{3,2T\}.
\]
The trivial block is consequently the bottleneck and
\(T_{\Phi_{\le2}}(S_d)=4\).

\paragraph{Representation stability for fixed-degree polynomial features.}
We conclude by proving the multiplicity formulas used in the main text.
Let \(p_d(r)\) denote the number of partitions of \(r\) into at most \(d\)
parts, and let \(p(r)\) be the unrestricted partition function.

For the homogeneous space \(V_r\), the multiplicity of the trivial
representation is
\[
\operatorname{mult}_{V_r}(\pi_0)
=
\dim V_r^{S_d}
=
p_d(r),
\]
since an orbit of monomials under coordinate permutations is specified by a
partition of \(r\) with at most \(d\) parts. Hence, for \(d\ge r\),
\[
\operatorname{mult}_{V_r}(\pi_0)=p(r).
\]

To compute the standard multiplicity, use the decomposition of the permutation
representation
\(\mathbb C^d=\pi_0\oplus\pi_{\mathrm{std}}\) and the identity
\[
\mathbb C^d
\cong
\operatorname{Ind}_{S_{d-1}}^{S_d}\mathbf 1.
\]
Frobenius reciprocity gives
\[
\operatorname{mult}_{V_r}(\mathbb C^d)
=
\dim V_r^{S_{d-1}}.
\]
An \(S_{d-1}\)-invariant polynomial may have an arbitrary power on one
distinguished variable and must be symmetric in the remaining \(d-1\)
variables. Therefore
\[
\dim V_r^{S_{d-1}}
=
\sum_{j=0}^r p_{d-1}(j).
\]
Subtracting the trivial multiplicity yields
\[
\operatorname{mult}_{V_r}(\pi_{\mathrm{std}})
=
\sum_{j=0}^r p_{d-1}(j)-p_d(r).
\]
When \(d\ge r+1\), this simplifies to
\[
\operatorname{mult}_{V_r}(\pi_{\mathrm{std}})
=
\sum_{j=0}^{r-1}p(j).
\]

Summing over \(0\le r\le k\), and taking \(d\ge k+1\), gives the stable
multiplicities
\[
m_{\pi_0}
=
\sum_{r=0}^k p(r),
\qquad
m_{\pi_{\mathrm{std}}}
=
\sum_{j=0}^{k-1}(k-j)p(j).
\]
Since the state permutation representation contains only these two irreducible
representations, defining
\[
M_k
\coloneqq
\max\{m_{\pi_0},m_{\pi_{\mathrm{std}}}\}
\]
and applying the preceding polynomial identification bound gives
\[
T_{\Phi_{\le k}}(S_d)\le M_k,
\qquad d\ge k+1.
\]
For fixed \(k\), \(M_k\) is independent of \(d\), proving
\(T_{\Phi_{\le k}}(S_d)=O_k(1)\).

Finally, the classical Hardy--Ramanujan estimate
\[
p(k)\sim
\frac{1}{4k\sqrt3}
\exp\left(\pi\sqrt{\frac{2k}{3}}\right)
\]
also gives the stated dependence on \(k\). Indeed,
\(M_k\ge m_{\pi_0}\ge p(k)\), while monotonicity of \(p(\cdot)\) gives
\(m_{\pi_0}\le(k+1)p(k)\) and
\(m_{\pi_{\mathrm{std}}}\le k^2p(k)\). Hence
\[
M_k
=
\exp\left(
\pi\sqrt{\frac{2k}{3}}+O(\log k)
\right).
\]

\subsection{Proofs for adaptive symmetry discovery}
\label{app:adaptive}

We first record a simple observation showing that it is sufficient to impose
equivariance on a generating set.

\begin{lemma}[Equivariance from generators]
\label{lem:equivariance_generators}
Let \(S\subseteq G\) generate \(G\). A matrix \(W\) satisfies
\[
\rho(s)W=W\rho_\Phi(s)
\qquad
\text{for every }s\in S
\]
if and only if it is \(G\)-equivariant.
\end{lemma}

\begin{proof}
Only the forward direction requires proof. If the intertwining identity holds
for \(g,h\in G\), then
\[
\rho(gh)W
=\rho(g)\rho(h)W
=\rho(g)W\rho_\Phi(h)
=W\rho_\Phi(gh).
\]
Moreover, \(\rho(g)W=W\rho_\Phi(g)\) also implies
\(\rho(g^{-1})W=W\rho_\Phi(g^{-1})\). Hence the identity is preserved under
composition and inversion, and therefore holds for every element of
\(\langle S\rangle=G\).
\end{proof}

\begin{proof}[Proof of \cref{thm:adaptive}]
We prove correctness, the trajectory-length guarantee, and computational
efficiency of Algorithm~\ref{alg:adaptive}.

\paragraph{Linear feasibility problem.}
Fix a candidate group \(H\in\mathcal G\) and sampled elements
\(S_H=\{g_1,\ldots,g_N\}\). From the observed trajectory, define
\[
X=[\Phi(x_0),\ldots,\Phi(x_{T-1})],
\qquad
Y=[x_1,\ldots,x_T].
\]
The algorithm asks whether there exists a candidate parameter matrix
\(\widetilde W\in\mathbb C^{d\times m}\) satisfying
\[
\widetilde W X=Y,
\qquad
\rho(g_i)\widetilde W
=
\widetilde W\rho_\Phi(g_i),
\quad i=1,\ldots,N.
\]
These are linear equations in the entries of \(\widetilde W\).

For completeness, letting
\(\widetilde w=\operatorname{vec}(\widetilde W)\), the data constraint becomes
\[
(X^\top\otimes I_d)\widetilde w=\operatorname{vec}(Y),
\]
while each equivariance constraint becomes
\[
\left(
I_m\otimes\rho(g_i)
-
\rho_\Phi(g_i)^\top\otimes I_d
\right)\widetilde w=0.
\]
Thus each candidate group is tested by solving a single finite-dimensional
linear feasibility problem.

\paragraph{Random generators.}
Let \(\mathcal E\) denote the event that, for every candidate
\(H\in\mathcal G\), the sampled multiset \(S_H\) generates \(H\).
By Corollary~\ref{cor:simultaneous_generators} and the choice
\(N=N_{\mathcal G,\delta}\),
\[
\mathbb P(\mathcal E)\ge1-\delta.
\]
We condition on this event for the remainder of the correctness argument.

On \(\mathcal E\), Lemma~\ref{lem:equivariance_generators} implies that the
equivariance constraints imposed for a candidate \(H\) are equivalent to full
\(H\)-equivariance. Hence the feasibility test succeeds for \(H\) exactly when
there exists an \(H\)-equivariant system consistent with the observed
trajectory.

\paragraph{Feasibility of the true group.}
Let \(G\in\mathcal G\) be the full symmetry group of the unknown dynamics
within the ambient group \(\Gamma\), and let \(W\) denote its parameter
matrix. Since the true dynamics are \(G\)-equivariant, \(W\) satisfies all
constraints corresponding to \(G\), and hence \(G\) is feasible.

Moreover, \(T\ge T_\Phi(\mathcal G)\ge T_\Phi(G)\). By
\cref{thm:equivariant_identification}, a generic \(G\)-equivariant system is
uniquely identifiable from the observed trajectory. Therefore \(W\) is the
unique \(G\)-equivariant parameter matrix consistent with the data. In
particular, any feasible solution stored by the algorithm for the candidate
\(G\) is exactly the true parameter matrix \(W\).

\paragraph{Characterization of all feasible candidates.}
Every candidate subgroup \(H\) of \(G\) is feasible, since a
\(G\)-equivariant system is automatically \(H\)-equivariant.

Conversely, suppose a candidate \(H\in\mathcal G\) is feasible. On the event
\(\mathcal E\), feasibility means that there exists an \(H\)-equivariant
system consistent with the trajectory. Since \(\mathcal G\) is generically
separating, Definition~\ref{def:candidate_separation} then implies that \(H\) is a
subgroup of \(G\).

Thus, on \(\mathcal E\), every feasible candidate is a subgroup of the true
group, while \(G\) itself is feasible. Since every proper subgroup of a finite
group has strictly smaller cardinality, \(G\) is the unique feasible candidate
of maximum cardinality. Algorithm~\ref{alg:adaptive} therefore selects \(G\)
and returns the corresponding parameter matrix \(W\).

The same event also guarantees
\(G=\langle S_G\rangle\), so the returned set \(S_G\) is a generating set of
the full symmetry group. Since \(\mathbb P(\mathcal E)\ge1-\delta\), the
algorithm succeeds with probability at least \(1-\delta\).

\paragraph{Trajectory length.}
The preceding argument requires only
\[
T\ge
T_\Phi(\mathcal G)
=
\max_{H\in\mathcal G}T_\Phi(H).
\]
Thus, under generic candidate separation, adaptive symmetry discovery requires
no additional trajectory observations beyond the worst-case trajectory length
needed when the candidate group is known.

\paragraph{Computational complexity.}
Let \(p=dm\). For each candidate group, the unknown vector
\(\widetilde w\) has \(p\) entries. The trajectory contributes \(dT\) scalar
linear equations, while the \(N\) sampled group elements contribute at most
\(Np\) scalar equivariance equations. Thus the complete feasibility system has
at most
\[
dT+Np
\]
equations in \(p\) unknowns.

Using dense linear algebra, feasibility and a solution can therefore be
computed in time polynomial in \(d,m,T\), and \(N\). More explicitly, forming
the equivariance constraints explicitly and applying Gaussian elimination
gives the coarse per-candidate bound
\[
O\bigl(d^3m^2T+Nd^3m^3\bigr),
\]
with polynomial memory complexity. Since
\[
N
=
O\left(
\log|G|_{\max}
+
\log|\mathcal G|
+
\log\frac1\delta
\right),
\]
and the procedure is repeated for \(|\mathcal G|\) candidates, the overall
runtime is polynomial in
\[
d,\;m,\;T,\;|\mathcal G|,\;
\log|G|_{\max},\;\log(1/\delta),
\]
assuming uniform sampling and representation evaluation for each candidate
group can be performed efficiently. This proves the theorem.
\end{proof}

\paragraph{Nested candidate families.}
We next justify the claim that generic candidate separation is automatic for
nested candidate families.

\begin{proposition}[Separation for nested families]
\label{prop:nested_separation}
Suppose the groups in \(\mathcal G\) are totally ordered by inclusion. Then
\(\mathcal G\) is generically separating.
\end{proposition}

\begin{proof}
Let \(G\in\mathcal G\) be the full symmetry group of a generic system and
suppose \(T\ge T_\Phi(\mathcal G)\). Consider a candidate
\(H\in\mathcal G\) admitting an \(H\)-equivariant system
\(\widetilde W\) consistent with the trajectory.

Since the family is totally ordered by inclusion, either \(H\) is a subgroup
of \(G\), in which case there is nothing to prove, or \(G\) is a proper
subgroup of \(H\). In the latter case, every \(H\)-equivariant system is also
\(G\)-equivariant. Hence \(\widetilde W\) is a \(G\)-equivariant system
consistent with the trajectory. Since
\(T\ge T_\Phi(G)\), \cref{thm:equivariant_identification} implies
\(\widetilde W=W\), where \(W\) is the true parameter matrix. But then the
true system is \(H\)-equivariant, contradicting the assumption that its full
symmetry group is \(G\). Therefore \(H\) must be a subgroup of \(G\).
\end{proof}

\subsection{Proof of \cref{cor:bounded_index}}
\label{app:bounded_index}

\begin{proof}
Let \(H\le\Gamma\) be the full symmetry group of the unknown dynamics, with
\([\Gamma:H]\le B\), and suppose generic elementwise separation holds at the
trajectory length \(T\).

Algorithm~\ref{alg:bounded_index} repeatedly samples
\(g\sim\operatorname{Unif}(\Gamma)\) and accepts \(g\) whenever there exists
\(\widetilde W\) satisfying
\[
\widetilde W X=Y,
\qquad
\rho(g)\widetilde W=\widetilde W\rho_\Phi(g).
\]
If \(g\in H\), the true parameter matrix \(W\) satisfies these constraints.
Conversely, generic elementwise separation implies that feasibility can hold
only if \(g\in H\). Hence the algorithm accepts exactly the samples lying in
\(H\).

It follows that the accepted elements are i.i.d.\ uniform samples from \(H\).
Moreover,
\[
\mathbb P_{g\sim\operatorname{Unif}(\Gamma)}(g\in H)
=
\frac{|H|}{|\Gamma|}
=
\frac{1}{[\Gamma:H]}
\ge\frac1B.
\]
Thus each accepted sample requires at most \(B\) ambient samples in
expectation.

By the random-generation result, the choice
\[
N_{\Gamma,\delta}
=
\left\lceil
8\left(
\left\lceil\log_2|\Gamma|\right\rceil
+\log\frac1\delta
\right)
\right\rceil
\]
ensures that the accepted set \(S\) generates \(H\) with probability at least
\(1-\delta\), since \(|H|\le|\Gamma|\).

Condition on this event. By Lemma~\ref{lem:equivariance_generators}, imposing the
intertwining constraints for all \(g\in S\) is equivalent to imposing
\(H\)-equivariance. Since
\[
T\ge T_\Phi(\mathcal H_B)\ge T_\Phi(H),
\]
\cref{thm:equivariant_identification} implies that the unique
\(H\)-equivariant parameter matrix consistent with the trajectory is the true
matrix \(W\). Therefore Algorithm~\ref{alg:bounded_index} returns the true
dynamics together with a generating set of its full symmetry group.

Finally, obtaining \(N_{\Gamma,\delta}\) accepted elements requires at most
\(BN_{\Gamma,\delta}\) ambient samples and feasibility tests in expectation.
Hence the expected number of tests is
\[
O\left(
B\left(\log|\Gamma|+\log\frac1\delta\right)
\right),
\]
with no dependence on the number of bounded-index subgroups of \(\Gamma\).
Assuming efficient ambient-group sampling, representation evaluation, and
linear feasibility testing, the expected runtime is polynomial in
\[
d,\;m,\;T,\;B,\;\log|\Gamma|,\;\log(1/\delta).
\]
If \(B=O(1)\), the rejection-sampling overhead is constant relative to direct
sampling from \(H\).

The same argument gives a high-probability runtime bound: by a standard
Chernoff bound, \(O(B(N_{\Gamma,\delta}+\log(1/\eta)))\) ambient samples
suffice to obtain \(N_{\Gamma,\delta}\) accepted elements with probability at
least \(1-\eta\).
\end{proof}

\begin{remark}[Cayley-graph interpretation]
The proof only requires that the accepted samples generate the unknown group
\(H\), equivalently that their symmetric Cayley graph is connected. The
Alon--Roichman theorem gives the stronger conclusion that
\(O(\log|H|)\) uniform samples produce an expanding Cayley graph with high
probability. In the bounded-index setting, the feasibility test acts as
rejection sampling from \(\Gamma\), producing exactly uniform elements of the
unknown subgroup \(H\), so the same random-Cayley-graph interpretation
applies.
\end{remark}

\section{Experiments}
\label{sec:experiments}

\begin{table}[t]
\centering
\caption{Dimension of the space of equivariant linear maps
\(\{A\in\mathbb{R}^{d\times d}:\rho(g)A=A\rho(g),\ \forall g\in G\}\)
for several permutation group actions on \(\mathbb{R}^d\), with \(d=10\).}
\label{tab:equivariant_dimensions}
\begin{tabular}{lccc}
\toprule
Group \(G\) &
Generators &
\(\dim(\text{equivariant maps})\) &
Description \\
\midrule
Trivial (no symmetry)
& \(\{e\}\)
& \(d^2=100\)
& All linear maps \\

Single transposition \(C_2=\langle(1\,2)\rangle\)
& one swap
& \(d^2-2d+2=82\)
& Single swap constraint \\

Cyclic shifts \(C_d\)
& one \(d\)-cycle
& \(d=10\)
& Circulant matrices \\

Dihedral group \(D_d\)
& shift + reversal
& \(d/2+1=6\)
& Symmetric circulant \\

Block permutations \(S_5\times S_5\)
& within-block swaps
& \(6\)
& Two exchangeable blocks \\

Full symmetric group \(S_d\)
& adjacent swaps
& \(2\)
& \(\alpha I+\beta\mathbf{1}\mathbf{1}^\top\) \\
\bottomrule
\end{tabular}
\end{table}

In this section, we present a simple proof-of-concept experiment illustrating
the sample-complexity results developed in this paper. Our focus is primarily
theoretical, and the experiment serves as a complementary empirical
illustration of the predicted identification thresholds.

\paragraph{Dimension of equivariant maps for permutation symmetries.}
We consider linear dynamics on \(\mathbb{R}^d\) of the form
\begin{equation}
x_{t+1}=Ax_t,\qquad t=0,\ldots,T-1,
\label{eq:lin_dyn_exp}
\end{equation}
where \(A\in\mathbb{R}^{d\times d}\) is unknown. Let \(G\) be a finite group
acting on \(\mathbb{R}^d\) by coordinate permutations, represented by
permutation matrices \(\rho(g)\). The dynamics is \(G\)-equivariant if
\begin{equation}
\rho(g)A=A\rho(g),\qquad \forall g\in G.
\label{eq:equivariance_lin}
\end{equation}
The matrices satisfying \eqref{eq:equivariance_lin} form the commutant
\[
\mathcal C(G)
\coloneqq
\{A\in\mathbb{R}^{d\times d}:
\rho(g)A=A\rho(g),\ \forall g\in G\}.
\]

For permutation actions, \(\dim(\mathcal C(G))\) has a simple combinatorial
interpretation. Equation~\eqref{eq:equivariance_lin} is equivalent to
\(A_{ij}=A_{g(i)g(j)}\) for every \(g\in G\). Thus, the entries of \(A\)
must be constant on the orbits of the diagonal action of \(G\) on ordered
pairs \((i,j)\). Consequently,
\begin{equation}
\dim(\mathcal C(G))
=
\#\{\text{orbits of \(G\) on }\{1,\ldots,d\}^2\}.
\label{eq:orbit_dim}
\end{equation}

Table~\ref{tab:equivariant_dimensions} gives this dimension for several
standard permutation actions with \(d=10\). For example, under the full
symmetric group \(S_d\), ordered pairs form only two orbits: diagonal pairs
\((i,i)\) and off-diagonal pairs \((i,j)\), \(i\neq j\). Hence
\(\dim(\mathcal C(S_d))=2\), and every \(S_d\)-equivariant matrix has the form
\(\alpha I+\beta\mathbf{1}\mathbf{1}^\top\), for
\(\alpha,\beta\in\mathbb{R}\).

\paragraph{Dimension of the feasible set vs.\ trajectory length.}
Given a trajectory \((x_0,\ldots,x_T)\) generated by
\eqref{eq:lin_dyn_exp}, define
\[
X=[x_0,\ldots,x_{T-1}],
\qquad
Y=[x_1,\ldots,x_T].
\]
The trajectory constraint is \(Y=AX\). Fix an assumed permutation symmetry
group \(H\) and restrict \(A\) to \(\mathcal C(H)\). We study the feasible set
\begin{equation}
\mathcal S_T(H)
\coloneqq
\{A\in\mathcal C(H):Y=AX\}.
\label{eq:feasible_set_def}
\end{equation}
Whenever nonempty, \(\mathcal S_T(H)\) is an affine subspace. We measure its
affine dimension \(\dim(\mathcal S_T(H))\); in particular,
\[
\dim(\mathcal S_T(H))=0
\quad\Longleftrightarrow\quad
\text{the trajectory uniquely identifies \(A\) within \(\mathcal C(H)\)}.
\]

\paragraph{Computing the feasible-set dimension.}
Let \(r=\dim(\mathcal C(H))\), and fix a basis
\(A_1,\ldots,A_r\) of \(\mathcal C(H)\). Writing
\(A=\sum_{i=1}^r\theta_iA_i\), the constraint \(Y=AX\) becomes
\[
\operatorname{vec}(Y)
=
\begin{bmatrix}
\operatorname{vec}(A_1X)&\cdots&\operatorname{vec}(A_rX)
\end{bmatrix}\theta.
\]
When feasible, \(\dim(\mathcal S_T(H))\) is therefore \(r\) minus the rank of
this design matrix.

\paragraph{Experimental setup.}
We set \(d=10\) and draw \(x_0\sim\mathcal N(0,I_d)\). We consider three
matched settings in which the true dynamics matrix \(A\) is drawn from a
continuous distribution on: (i) \(\mathcal C(C_2)\), where
\(C_2=\langle(1\,2)\rangle\); (ii) the unconstrained space
\(\mathbb{R}^{d\times d}\); and (iii) \(\mathcal C(S_d)\). For each setting,
we evaluate \(\dim(\mathcal S_T(H))\) using the corresponding correct symmetry
assumption \(H\), and report the median over independent trials.

These three settings have exact theoretical identification thresholds predicted
by the linear-system result of \cref{sec:equivariant_identification}. For the
trivial group, the state representation contains \(d\) copies of its
one-dimensional irrep, giving \(T_{\mathrm{lin}}=d=10\). For the single
transposition, the state representation consists of \(d-1\) copies of the
trivial representation and one copy of the sign representation, giving
\(T_{\mathrm{lin}}=d-1=9\). Finally, for \(S_d\),
\(\mathbb{R}^d\cong\pi_0\oplus\pi_{\mathrm{std}}\), with both irreducible
components appearing once, and hence \(T_{\mathrm{lin}}=1\).

\paragraph{Observed behavior.}
Figure~\ref{fig:dim} plots \(\dim(\mathcal S_T(H))\) as a function of the
trajectory length \(T\). The feasible-set dimension decreases with \(T\) and
reaches zero exactly when the dynamics becomes uniquely identifiable within
the assumed equivariant class. The empirical transitions agree with the
theoretical predictions: the fully permutation-equivariant system reaches
dimension zero at \(T=1\), the single-transposition system at \(T=9\), and the
unconstrained system at \(T=10\).

The dimensions in Table~\ref{tab:equivariant_dimensions} provide useful
intuition for the reduction in parameter complexity under symmetry, although
the identification threshold is determined more precisely by the
representation multiplicities characterized in
\cref{thm:equivariant_identification}.

\begin{figure}[t]
    \centering
    \includegraphics[width=0.7\linewidth]{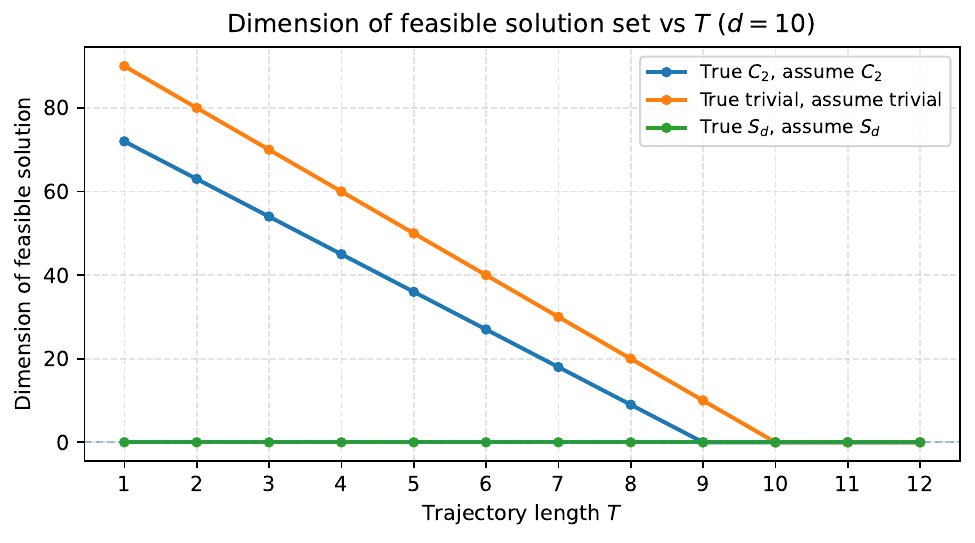}
    \caption{Dimension of the feasible solution set
    \(\mathcal S_T(H)\) as a function of trajectory length \(T\) for the
    correctly matched symmetry classes. The dimension reaches zero at the
    predicted generic identification thresholds \(T=1,9,\) and \(10\) for
    \(S_d\), \(C_2\), and the trivial group, respectively.}
    \label{fig:dim}
\end{figure}



\section{Notation}
\label{app:notation}

We collect the main notation used throughout the paper. 

\begin{table}[h]
\centering
\small
\begin{tabular}{ll}
\toprule
Symbol & Meaning \\
\midrule
\(G\) & Finite group acting on the state and feature spaces \\
\(\widehat G\) & Equivalence classes of irreducible representations of \(G\) \\
\(\pi\in\widehat G\) & An irreducible representation of \(G\) \\
\(V_\pi\) & Representation space of \(\pi\) \\
\(d_\pi\) & Dimension of \(V_\pi\) \\
\(\rho\) & Representation of \(G\) on the state space \(\mathbb C^d\) \\
\(\rho_\Phi\) & Representation of \(G\) on the feature space \(\mathbb C^m\) \\
\(n_\pi\) & Multiplicity of \(\pi\) in the state representation \\
\(m_\pi\) & Multiplicity of \(\pi\) in the feature representation \\
\(\Phi\) & Feature map \(\Phi:\mathbb C^d\to\mathbb C^m\) \\
\(\Phi_{\le k}\) & Full polynomial feature map of degree at most \(k\) \\
\(\mathcal F_\Phi\) & Systems \(f(x)=W\Phi(x)\) associated with \(\Phi\) \\
\(\mathcal F_\Phi^G\) & \(G\)-equivariant systems in \(\mathcal F_\Phi\) \\
\(W\) & Parameter matrix \(W:\mathbb C^m\to\mathbb C^d\) \\
\(C_\pi\) & Multiplicity-space matrix of the \(\pi\)-block of \(W\) \\
\(I_{V_\pi}\) & Identity map on \(V_\pi\) \\
\(x_t\) & State at time \(t\) \\
\(T\) & Trajectory length \\
\(X_\pi(x_t)\) & Matrix form of the \(\pi\)-isotypic component of \(x_t\) \\
\(\Phi_\pi(x_t)\) & Matrix form of the \(\pi\)-isotypic component of \(\Phi(x_t)\) \\
\(\mathbf{\Phi}_{\pi,T}\) &
Block feature matrix \([\Phi_\pi(x_0),\ldots,\Phi_\pi(x_{T-1})]\) \\
\(h_{\pi,\Phi}(T)\) & Generic rank of \(\mathbf{\Phi}_{\pi,T}\) \\
\(T_\Phi(G)\) & Generic identification trajectory length for known \(G\) \\
\(T_{\mathrm{rep}}(G)\) &
Representation-theoretic lower bound on \(T_\Phi(G)\) \\
\(\Gamma\) & Ambient finite group of candidate transformations \\
\(\mathcal G\) & Finite family of candidate groups \\
\(T_\Phi(\mathcal G)\) &
Worst-case threshold \(\max_{G\in\mathcal G}T_\Phi(G)\) \\
\(|G|_{\max}\) & Maximum cardinality of a group in \(\mathcal G\) \\
\(S\) & Set or multiset of sampled group elements \\
\(\langle S\rangle\) & Subgroup generated by \(S\) \\
\(\operatorname{Cay}(G,S)\) & Cayley graph of \(G\) generated by \(S\) \\
\(N_{\mathcal G,\delta}\) &
Number of random samples used per candidate group \\
\(\delta\) & Failure probability \\
\(\mathcal H_B\) & Family of candidate subgroups of index at most \(B\) \\
\([\Gamma:H]\) & Index of a subgroup \(H\) in \(\Gamma\) \\
\(B\) & Upper bound on candidate subgroup index \\
\(M_k\) & Stable maximal active multiplicity for degree-\(k\) features \\
\(p(r)\) & Number of integer partitions of \(r\) \\
\(\mathcal C(G)\) & Commutant of a permutation action in the experiments \\
\(\mathcal S_T(H)\) & Feasible parameter set under assumed symmetry \(H\) \\
\bottomrule
\end{tabular}
\label{tab:notation}
\end{table}


\end{document}